\documentclass[twoside]{article}
\usepackage[accepted]{aistats2026}

\usepackage[utf8]{inputenc}
\usepackage[T1]{fontenc}    
\usepackage{lmodern}
\usepackage{booktabs}       
\usepackage{nicefrac}       
\usepackage{microtype}      
\usepackage{xcolor}         
\usepackage{bm}
\usepackage{url}
\usepackage{amsmath}
\usepackage{amsthm}
\usepackage{amsfonts}
\usepackage{cancel}
\usepackage{graphicx}
\usepackage{natbib}
\usepackage{multirow}
\usepackage{makecell}
\usepackage{enumitem}
\usepackage{tabu}
\usepackage{wrapfig}
\usepackage{todonotes}
\usepackage{dblfloatfix}
\usepackage{float}

\usepackage{algorithm}
\usepackage{algpseudocode}

\usepackage{hyperref}

\newtheorem{thm}{Theorem}
\newtheorem{prop}{Proposition}

\newtheorem{lemma}{Lemma}

\newtheorem{remark}{Remark}

\DeclareMathOperator*{\argmax}{arg\,max}
\DeclareMathOperator*{\argmin}{arg\,min}

\newcommand{\R}{\mathbb{R}}
\newcommand{\E}{\mathbb{E}}
\newcommand{\I}{\bm{I}}

\newcommand{\Ph}{\bm{\Phi}}

\newcommand{\phsv}{\bm{\varphi}^*}
\newcommand{\phsvt}{\bm{\varphi}^{*\top}}

\newcommand{\nv}{\bm{0}}
\newcommand{\Tr}{\text{\normalfont Tr}}

\newcommand{\sgn}{\text{\normalfont sign}}
\newcommand{\diag}{\text{\normalfont diag}}
\newcommand{\epsv}{\bm{\varepsilon}}
\newcommand{\fv}{\bm{f}}

\newcommand{\xv}{\bm{x}}
\newcommand{\xsv}{\bm{x}^*}
\newcommand{\xsvt}{\bm{x}^{*\top}}
\newcommand{\xvi}{\bm{x_i}}
\newcommand{\xpv}{\bm{x'}}
\newcommand{\X}{\bm{X}}
\newcommand{\Xs}{\bm{X}_v^*}
\newcommand{\yv}{\bm{y}}
\newcommand{\yvr}{\bm{y}_R}
\newcommand{\K}{\bm{K}}
\newcommand{\Kt}{\bm{\tilde{K}}}
\newcommand{\Kntk}{\bm{K}^{\star\text{\normalfont NTK}}}
\newcommand{\Ss}{\bm{S}}
\newcommand{\Sss}{\bm{S}_v^*}
\newcommand{\Ssst}{\bm{S}_v^{*\top}}

\newcommand{\ssv}{\bm{s}^*}
\newcommand{\ssvt}{\bm{s}^{*\top}}
\newcommand{\thetahv}{\bm{\hat{\theta}}}
\newcommand{\thetah}{\hat{\theta}}
\newcommand{\thetav}{\bm{\theta}}
\newcommand{\fhv}{\bm{\hat{f}}}
\newcommand{\fhs}{\hat{f}^*}
\newcommand{\fhsv}{\bm{\hat{f}}_v^*}
\newcommand{\ns}{n^*_v}
\newcommand{\fhiv}{\bm{\hat{f}_i}}
\newcommand{\Lt}{\tilde{L}}
\newcommand{\yiv}{\bm{y_i}}
\newcommand{\Fh}{\bm{\hat{F}}}
\newcommand{\Fhi}{\bm{\hat{F}_i}}
\newcommand{\Y}{\bm{Y}}
\newcommand{\fh}{\hat{f}}
\newcommand{\vv}{\bm{v}}
\newcommand{\A}{\bm{A}}
\newcommand{\B}{\bm{B}}
\newcommand{\N}{\mathcal{N}}
\newcommand{\kvt}{\bm{\tilde{k}}}
\newcommand{\IFk}{\bm{I}_{\bm{\tilde{F}}_k}}

\newcommand{\ehv}{\bm{\hat{e}}}
\newcommand{\eh}{\hat{e}}

\newcommand{\snr}{\text{\normalfont SNR}}

\newcommand{\Sigh}{\bm{\hat{\Sigma}}}
\newcommand{\Sig}{\bm{\Sigma}}

\newcommand{\asymrisk}{\overline{R_{\bm{X}}}}
\newcommand{\ssmmrisk}{\overline{R_{\bm{X}}}(\overline{\lambda_T})}
\newcommand{\optrisk}{\overline{R_{\bm{X}}}(\overline{\lambda^*})}
\newcommand{\zerorisk}{\overline{R_{\bm{X}}}(0)}

\begin{document}

\twocolumn[

\aistatstitle{Is Supervised Learning Really That Different From Unsupervised?}

\aistatsauthor{Oskar Allerbo \And Thomas B. Schön}

\aistatsaddress{KTH Royal Institute of Technology \And  Uppsala University } ]

\begin{abstract}
We demonstrate how supervised learning can be decomposed into a two-stage procedure, where (1) all model parameters are selected in an unsupervised manner, and (2) the outputs~$\yv$ are added to the model, \emph{without changing the parameter values}. This is achieved by a new model selection criterion that---in contrast to cross-validation---can be used also without access to~$\yv$. For linear ridge regression, we bound the asymptotic out-of-sample risk of our method in terms of the optimal asymptotic risk. We also demonstrate that versions of linear and kernel ridge regression, smoothing splines, k-nearest neighbors, random forests, and neural networks, trained without access to $\yv$, perform similarly to their standard $\yv$-based counterparts. Hence, our results suggest that the difference between supervised and unsupervised learning is less fundamental than it may appear.
\end{abstract}





\begin{figure}[t]
\center
\includegraphics[width=\columnwidth]{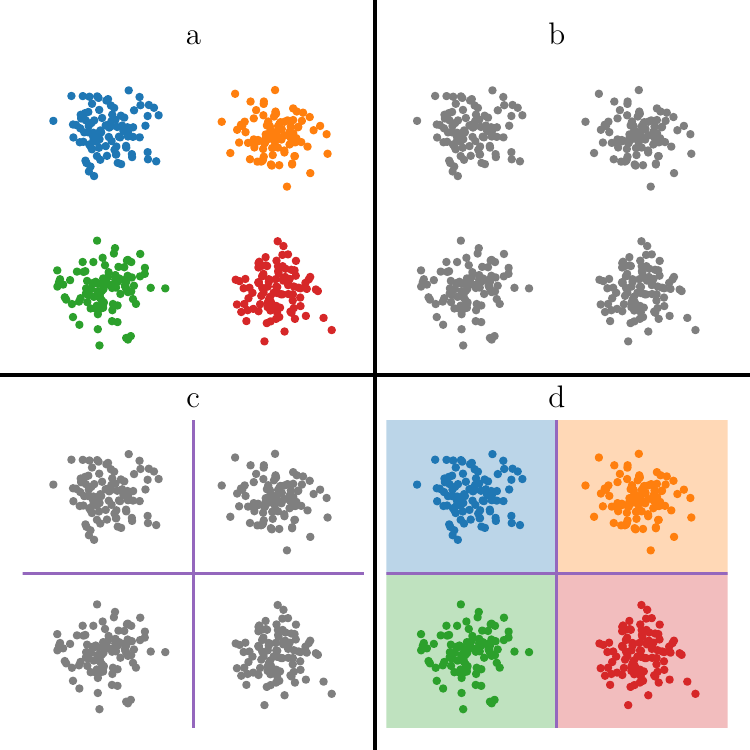}
\caption{Illustrating training without labels: After the label information ($a$) has been removed ($b$), an unsupervised algorithm is used to split the input space into four classes, separated by the purple cross, without using the label information ($c$). After fixing the class borders, the labels are revealed, enabling the model to label each class ($d$).}
\label{fig:cl_demo}
\vspace{-.3cm}
\end{figure}

\section{INTRODUCTION}
In contrast to supervised learning, unsupervised learning is performed without access to the outputs $\yv$, i.e.\ the labels (for classification) or the response (for regression). The success of unsupervised learning raises the question: Are the outputs really needed when training supervised models? Or is it possible to successfully train a supervised model without $\yv$, not including the labels/response until after all model parameters are selected and kept fixed? If this is indeed the case, it suggests that---under the hood---supervised learning is closely related to unsupervised learning.

In Figure \ref{fig:cl_demo}, we give an example where supervised classification can be replaced by a two-stage procedure, based on unsupervised learning. First, the model parameters are selected in an unsupervised manner, splitting the input data into four classes, with (yet) unknown labels. Then, the training labels are revealed, allowing the model to assign labels to the four classes, but without altering the borders between classes.

More generally, the key to training without $\yv$ is to express the model as a smoother, i.e.\ on the form $\fhs=\ssvt\yv$, where $\fhs=\fh(\xsv)$ denotes the prediction for covariate $\xsv$, and $\ssv$ denotes a smoother vector. This implies that each prediction is a weighted sum of the $n$ observations in $\yv$, where the weights are given by the smoother vector. 
If~$\ssv$ only depends on the covariates and not on $\yv$, the model is linear in $\yv$ (but \emph{not} necessarily in $\xv$). Such a model is referred to as a linear smoother. If $\ssv$ additionally depends on $\yv$, it is called a nonlinear smoother. For many models, including linear and kernel ridge regression, $\ssv$ follows immediately from the closed-form expression of the model, while for others, it is more complicated, but still possible. For instance, \citet{jeffares2024deep} recently showed how neural networks can be expressed as nonlinear smoothers based on the neural tangent kernel (NTK) framework. Since the smoother contains all the model parameters, whenever $\ssv$ is selected independently of~$\yv$, training the model is $\yv$-free. 

The possibility of training without $\yv$ has an interesting implication: It suggests that training a supervised model amounts to finding a good representation of $\xsv$, and that this can be done without the information in $\yv$. To obtain the predictions, $\yv$ must be included in the model; however, it is sufficient to do this \emph{after} the representation of $\xsv$ is constructed. It can even be hypothesized that the optimal model can always be written as a linear smoother, meaning that the optimal predictions are weighted sums of $\yv$. The weights depend on the training covariates $\X$ and $\xsv$. Hence, the weights depend on how the out-of-sample data relate to the training data, but not on $\yv$.

In this work, we introduce a model-independent way to train supervised models, including neural networks, as linear smoothers. We also demonstrate that the obtained models compare very well to the corresponding models trained with the information in $\yv$ in terms of performance. We thus illustrate both the close relation between supervised and unsupervised learning and the plausibility of the hypothesis above.

Our \textbf{main contribution} is to \textbf{demonstrate} that \textbf{standard supervised models} for classification and regression can be \textbf{successfully trained in an unsupervised manner}. We do this by
\begin{itemize}[leftmargin=6mm]
\item
introducing a \emph{model selection criterion} that is based on the out-of-sample model predictions, and thus works for any classification/regression model. We demonstrate that this criterion, in contrast to cross-validation, also \emph{works without access to $\yv$}.
\item
bounding the asymptotic risk of our method in terms of the optimal asymptotic risk for linear ridge regression.
\end{itemize}

The aim of this paper is to demonstrate the feasibility of training supervised models without access to~$\yv$, thereby highlighting the parallels between supervised and unsupervised learning. It is not intended to suggest that our method outperforms existing $\yv$-based approaches, nor that models should be trained without~$\yv$ in practical applications.

All proofs are deferred to Appendix~\ref{sec:proofs} and the code is available at \url{https://github.com/allerbo/training_without_y}.

\section{NOTATION}
We denote scalars with lowercase non-bold letters, vectors with lowercase bold letters, and matrices with uppercase bold letters. The in-sample (training) data are denoted as $\X\in\R^{n\times d}$ and $\yv\in \R^n$. The out-of-sample quantities are denoted with a star: $\xsv\in\R^d$ denotes any out-of-sample covariate, while $\Xs\in\R^{\ns \times d}$ denotes out-of-sample validation covariates. Predictions are denoted as $\fhv=\fhv(\X)\in\R^n$ (in-sample), $\fhs=\fhs(\xsv)\in\R$ (out-of-sample), and $\fhsv=\fhsv(\Xs)\in\R^{\ns}$ (validation), while the true function values are denoted as $\fv=\fv(\X)\in\R^n$ (in-sample) and $f^*=f^*(\xsv)\in\R$ (out-of-sample). $\Ph=\Ph(\X)\in\R^{n\times p}$ and $\phsv=\phsv(\xsv)\in\R^p$ denote p-dimensional feature expansions of the in- and out-of-sample data respectively, while $\Ss=\Ss(\X)\in\R^{n \times n}$, $\ssv=\ssv(\xsv,\X)\in\R^n$, and $\Sss=\Sss(\Xs,\X)\in\R^{\ns \times n}$ denote the in-sample, out-of-sample and validation smoother matrices (or vectors). $\I_d\in\R^{d\times d}$ denotes the identity matrix.

\section{RELATED WORK}
\textbf{Unsupervised Learning:}
Unsupervised learning comprises a broad class of methods that share the common goal of learning the structure of unlabeled data. Classical examples of unsupervised learning include \emph{clustering}, such as different flavors of k-means \citep{macqueen1967some,lloyd1982least}, spectral \citep{ng2001spectral,von2007tutorial}, density-based \citep{ester1996density,bhattacharjee2021survey}, and hierarchical \citep{murtagh2012algorithms,campello2013density} clustering; \emph{dimensionality reduction}, including different flavors of principal component analysis \citep{pearson1901liii, scholkopf1997kernel, zou2006sparse}, autoencoders \citep{kramer1991nonlinear,ng2011sparse, kingma2013auto, allerbo2021non} and graph-based methods \citep{maaten2008visualizing,mcinnes2018umap}; and different approaches to  (non-negative) \emph{matrix factorization} \citep{lee2000algorithms,mnih2007probabilistic,koren2009matrix}. Another example of unsupervised learning is (deep) \emph{generative models}, including variational autoencoders \citep{kingma2013auto}, generative adversarial networks \citep{goodfellow2014generative}, normalizing flows \citep{rezende2015variational}, and diffusion models \citep{sohl2015deep}, where the goal is to learn the distribution of the data, to enable sampling of realistic, synthetic data.

\textbf{Selfsupervised Learning:}
In self-supervised learning, rather than using externally provided training labels, the supervisory signal is constructed from the covariate data. This is often accomplished by augmenting or transforming the data to create new samples that are related by construction. Contrastive learning (see e.g. \citet{chen2020simple}, \citet{he2020momentum}, and \citet{wang2020understanding}) enforces data generated from the same instance to have similar latent representations, while instance discrimination (see e.g. \citet{wu2018unsupervised}, and \citet{ibrahim2024occam}) assigns pseudo-labels to the augmented data, where the data generated from the same original instance share labels. Using self-supervised methods for identifying the data-generating process has been studied by, e.g., \citet{zimmermann2021contrastive}, \citet{hyvarinen2019nonlinear}, and \citet{reizinger2024cross}, where the latter two use non-linear independent component analysis (see e.g.  \citet{hyvarinen2023nonlinear} for an overview). Our work differs from self-supervised learning in that we do not use any supervisory signal during training (not even one based on augmented data). Furthermore, we do not make any assumptions about a generating process that we try to identify.

\textbf{Model Selection without Validation Data:}
Hyperparameter selection is usually based on some sort of cross-validation, i.e.\ by training the model on parts of the data and validating it on other parts. There are, however, many methods for selecting hyperparameters without evaluating the performance on validation data, but, in contrast to our work, they still use the outputs~$\yv$. The most classic methods are probably Mallows's $C_p$ \citep{mallows1964choosing}, the Akaike information criterion (AIC) \citep{akaike1973information}, and the Bayesian information criterion (BIC) \citep{schwarz1978estimating}. They all estimate the in-sample prediction error, but since then, many alternatives have been proposed. These include approaches based on
the gradient of the loss function \citep{mahsereci2017early};
estimated marginal likelihood \citep{duvenaud2016early,lyle2020bayesian};
stability of predictions \citep{lange2002stability, yuan2025early};
the neural tangent kernel \citep{jacot2018neural,xu2021knas, chen2021neural, zhu2022generalization};
unlabeled test data \citep{garg2021ratt,deng2021labels,peng2024energy};
training speed \citep{ru2021speedy,zhang2024mote}; 
topological properties of the model \citep{corneanu2020computing,ballester2024predicting}; and
desired properties of the inferred function \citep{allerbo2022bandwidth}.

\textbf{Training with Random Labels:}
One way to make a model independent of $\yv$ is to replace $\yv$ with a random vector. This was investigated by \citet{zhang2016understanding}, who showed that overparameterized neural networks can achieve zero training error when trained with random labels. However, in contrast to our work, in terms of generalization, the resulting models were at the level of random guessing.
Since then, training on partly, or completely, random labels has received a lot of attention, where the main focus has been on partly random labels, see e.g.\
\citet{zhang2016understanding},
\citet{jiang2018mentornet},
\citet{zhang2018generalized},
\citet{han2019deep},
\citet{song2022learning},
\citet{wei2024vision},
\citet{nguyen2024noisy}, and
\citet{wang2024tackling}.
Training with completely random labels has successfully been used for
neural architecture search \citep{zhang2021neural};
estimating model complexity \citep{becker2024learned};
partitioning models across GPUs during training \citep{karadag2025partitioned};
and model pretraining \citep{pondenkandath2018leveraging,maennel2020neural,behrens2025knowledge}.

\textbf{Smoother Models:}
Linear smoothers have a long history and have been studied by e.g.\ 
\citet{watson1964smooth}, 
\citet{reinsch1967smoothing},
\citet{rosenblatt1971curve},
\citet{silverman1985some}, and
\citet{buja1989linear}.
Due to their more complicated form, nonlinear smoothers are harder to analyze theoretically but have been studied by e.g.\ \citet{mallows1980some}.
More recently, \citet{jeffares2024deep} showed that deep neural networks trained with the squared loss and infinitesimal learning rate can be formulated as smoothers, where the smoother matrix is a function of the neural tangent kernel, NTK. In some settings, the NTK is constant during training \citep{jacot2018neural,lee2019wide}, and thus independent of $\yv$, in which case the neural network is a linear smoother. However, in practice, the NTK changes during training \citep{liu2020linearity,fort2020deep}, making it dependent on $\yv$, and the neural network thus becomes a nonlinear smoother. This stands in contrast to our work, where the neural network smoother matrix is always independent of $\yv$. Due to the high computational cost of calculating the NTK, computationally efficient methods have been proposed by 
\cite{novak2022fast}, \citet{wei2022more}, \citet{wang2022deep}, and \citet{mohamadi2023fast}, where the last three rely on approximations of the kernel.

\textbf{In- and Out-of-Sample Performance:}
A classical measure of model capacity is the effective number of parameters, also known as the effective degrees of freedom, \citep{efron1986biased,efron2004estimation}. This is computed using only in-sample data, which limits its usefulness in some situations, see e.g.\ \citet{janson2015effective}. 
Generalizations of the effective number of parameters that also include out-of-sample data have been proposed by \citet{rosset2020fixed}, \citet{luan2021predictive}, \citet{curth2023u}, and \citet{patil2024revisiting}.
A related topic that has gained a lot of attention is when in- and out-of-sample data follow different distributions, see e.g.\ 
\citet{hendrycks2021many},
\citet{liu2021towards},
\citet{harun2024variables},
\citet{yang2024generalized}, and
\citet{cui2025online}.

\section{MODEL SELECTION WITHOUT USING THE OUTPUTS}
In this section, we first discuss why in‑sample model complexity is not necessarily a reliable basis for model selection. We then propose a model selection criterion based on the variance of the out‑of‑sample predictions, which we make independent of~$\yv$.

\subsection{Limitations of Model Selection Based on In-Sample Model Complexity}
Let us consider model selection by pre-defining the desired model complexity in terms of the effective number of parameters. For linear smoothers, where $\fhv=\Ss\yv$, the effective number of parameters is obtained as the trace of $\Ss$. The model complexity thus depends on $\X$ and the model parameters, but not on $\yv$ and $\xsv$.
Although often working relatively well, this method has some obvious limitations: First, it is not clear \emph{which complexity to choose} to obtain good generalization. Second, when there is more than one model parameter, there may exist several different models with the \emph{same complexity}, but with \emph{different performance}. This is illustrated in Figure \ref{fig:compl_demo}, where we have used kernel ridge regression with a Gaussian kernel to fit the data, which means that there are two parameters: regularization and kernel bandwidth. 
All three functions have the same in-sample complexity, and they all perfectly interpolate the training data. However, while the green function provides reasonable out-of-sample predictions, the red and yellow functions greatly over- or underfit between observations. Consequently, in-sample model complexity, even in a two-parameter model, is not sufficient to predict out‑of‑sample performance.

\subsection{Model Selection Based on Out-of-Sample Predictions: Controlling the Variance}
\label{sec:ssmm}

For the green function in Figure \ref{fig:compl_demo}, which generalizes well compared to the red and yellow ones, the out-of-sample predictions roughly follow the same distribution as the observations, i.e.\ $\fhs\stackrel{d}\approx y$ (where $\stackrel{d}\approx$ denotes approximately equal in distribution). This also makes sense: for a model to generalize well, the behavior of the predictions (including the out-of-sample predictions) must not deviate too much from that of the data. 

Given out-of-sample validation covariates $\Xs$, for which no $\yv$-data exists, we propose model selection by matching the empirical distributions of $\yv$ and $\fhsv$ through matching of moments. 
For centered data, $\frac1n\sum_{i=1}^ny_i=0$ and $\frac1{\ns}\sum_{i=1}^{\ns}\fhs_{v,i}\approx 0$, which means that the first sample moments already approximately match. We additionally require the second sample moments, and hence the sample variances, to match, disregarding higher-order moments for simplicity. 
For a smoother, model selection thus amounts to minimizing
\begin{equation}
\label{eq:ssmm}
\begin{aligned}
&\left|\frac1n\sum_{i=1}^ny_i^2-\frac1{\ns}\sum_{i=1}^{\ns}\fhs_{v,i}{^2}\right|\\
=&\left|\yv^\top\left(\frac1n\I_n-\frac1{\ns}\Ssst\Sss\right)\yv\right|,
\end{aligned}
\end{equation}
a method we refer to as matching of sample variances, MSV. So far, MSV still depends on $\yv$; we address this in Section \ref{sec:y-free}.

To evaluate MSV, we need to calculate the out-of-sample validation smoother matrix $\Sss$, for which we need $\Xs$. In principle, we would like to use all possible future $\xv$-values, which is, of course, unrealistic, and we need to settle for a sub-sample.
One way to create such a sample would be to split $\X$ into two parts and use one for training and the other for validation. However, this comes with at least two limitations: First, observations in $\yv$ corresponding to the validation covariates would then not be considered at all during training, which would be wasteful. Second, we would like the model to generalize well to all possible future values $\xv$, also those not previously seen. Both of these issues can be solved by sampling $\Xs$ from a generative model, trained on $\X$, something that allows us to create as large an $\Xs$ as we wish.
Throughout this paper, we generate $\Xs$ by sampling from a multivariate normal distribution, parameterized by the sample mean and covariance of $\X$. As we will see, even this extremely simple generative model works well in practice.
An alternative to sample $\Xs$ is to replace the second sample moment, $\frac1{\ns}\Ssst\Sss$, in Equation \ref{eq:ssmm} with the true second moment, $\E_{\xsv}[\ssv\ssvt]$. Whenever $\E_{\xsv}[\ssv\ssvt]$ can be obtained in closed form, which is not always the case, sampling of $\Xs$ is no longer needed.

\begin{figure}
\center
\includegraphics[width=\columnwidth]{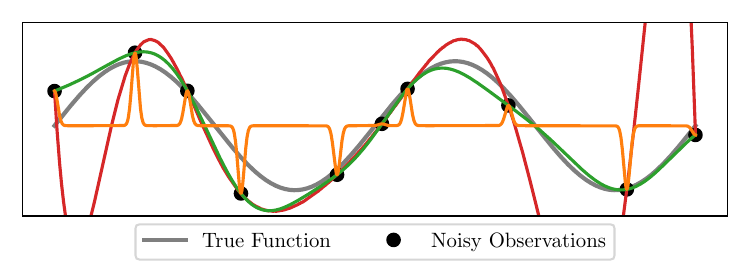}
\caption{Illustrating the limitations of in-sample model complexity. All three function estimates (green, orange, and red) have the same in-sample model complexity, in terms of the effective number of parameters, but behave very differently between the observations. Details are given in Appendix~\ref{sec:exp_dets}.}
\label{fig:compl_demo}
\end{figure}

\subsection{Making Model Selection \texorpdfstring{$\yv$}{y}-free}
\label{sec:y-free}
As can be seen from Equation~\ref{eq:ssmm}, model selection through MSV amounts to minimizing a closed-form expression of the form $|\yv^\top\A\yv|$, where $\A$ is a symmetric matrix that does not depend on $\yv$ unless $\Sss$ does. 
We can make this expression independent of $\yv$ by simply replacing $\yv$ with a random vector, $\yvr$. It can be shown (see Appendix \ref{sec:proofs}) that when the elements in $\yvr$ are uncorrelated, with zero mean and variance $\sigma_y^2$, then
\begin{equation}
\label{eq:t_semi_norm}
\left|\E_{y}\left(\yvr^\top\A\yvr\right)\right|
=\left|\Tr\left(\A\right)\right|\cdot\sigma_y^2
=:\left\|\A\right\|_T\cdot\sigma_y^2.
\end{equation}
We can thus obtain $\yv$-free model selection, by minimizing 
\begin{equation*}
\label{eq:ssmm_tr}
\left\|\frac1n\I_n-\frac1{\ns}\Ssst\Sss\right\|_T=\left|1-\frac1{\ns}\Tr\left(\Ssst\Sss\right)\right|,
\end{equation*}
or $\left|1-\Tr\left(\E(\ssv\ssvt)\right)\right|=\left|1-\E(\|\ssv\|_2^2)\right|$ if the expectation is available, something we refer to as MSV-Tr.\footnote{MSV-Tr is closely related to the generalized effective number of parameters introduced by \citet{curth2023u}: $\left|1-\frac1{\ns}\Tr\left(\Ssst\Sss\right)\right|=\left|1-n\cdot p_{\bm{\hat{S}}}^0\right|$, where $p_{\bm{\hat{S}}}^0$ denotes their generalization of the effective number of parameters.}
An alternative to Equation \ref{eq:t_semi_norm} for making MSV independent of $\yv$ is to note that
\begin{equation}
\label{eq:nf2_norm}
\begin{aligned}
\left|\yv^\top\A\yv\right|
&=\left\|\yv^\top\A\yv\right\|_2
\leq \left\|\A\right\|_2\cdot\left\|\yv\right\|^2_2\\
&\leq \left\|\A\right\|_F\cdot\left\|\yv\right\|^2_2
\leq \left\|\A\right\|_*\cdot\left\|\yv\right\|^2_2,
\end{aligned}
\end{equation}
where $\|\cdot\|_{2}$, $\|\cdot\|_{F}$, and $\|\cdot\|_{*}$ denote the spectral, Frobenius, and nuclear matrix norms, respectively. The expressions in Equations \ref{eq:t_semi_norm} and \ref{eq:nf2_norm} are very similar, but differ in one important aspect: In contrast to the three proper norms, unless $\A$ is positive semi-definite (PSD), $\|\A\|_T=|\Tr(\A)|$ is a seminorm, which means that it may be 0 for matrices other than the zero matrix. When $\A$ is PSD, its eigenvalues and singular values coincide, and $\|\A\|_T=\|\A\|_*$.

\section{WHEN DOES \texorpdfstring{$\yv$}{y}-FREE MODEL SELECTION WORK?}
In this section, we analyse the MSV-Tr model selection criterion in terms of the bias-variance decomposition of the conditional out-of-sample prediction risk. 
We also compare the asymptotic risk, as $n,d\to\infty$, for linear ridge regression with regularization selected by MSV-Tr to the theoretically optimal risk.
We finally show that, in contrast to MSV, $\yv$-free generalized cross-validation cannot be used for model selection. 

\subsection{Analysis of  \texorpdfstring{$\yv$}{y}-free MSV-Tr}
Assuming that the data is generated as $\yv=\fv+\bm{\varepsilon}$, where the elements in $\bm{\varepsilon}\in\R^n$ are i.i.d.\ with zero mean and variance $\sigma^2_\varepsilon$, the conditional out-of-sample risk is then given by
\begin{equation*}
\begin{aligned}
R_{\X}
&:=\E\left((\fhs-f^*)^2|\X\right)=B_{\X}+V_{\X},
\end{aligned}
\end{equation*}
with bias and variance components $B_{\X}$ and $V_{\X}$.

According to Proposition \ref{thm:mar_pas1}, MSV-Tr selects the model so that the variance of the inferred out-of-sample functions equals the variance of the measurement noise, which is very reasonable. Since we must not use $\yv$ when selecting the model parameters, the bias cannot be estimated, and thus the risk must be assessed only through its variance component. Since, in general, both a too small variance (which generally leads to a too high bias) and a too large variance lead to a high risk, using the same variance as the noise is often a good compromise. One may argue that $\I_n/n$ in Equation \ref{eq:ssmm} should be multiplied by some $a>0$ to obtain $|a\cdot\sigma_\varepsilon^2-V_{\X}|$ in Equation \ref{eq:mp1}. The model could then be made more robust (or flexible) by tuning $a$. This would, however, require estimating the signal-to-noise ratio (SNR), for which we need access to $\yv$.
\begin{prop}
\label{thm:mar_pas1}
For a linear smoother $\fhs=\ssv(\thetahv)^\top\yv$, with parameters $\thetahv$,
\begin{equation}
\label{eq:mp1}
\begin{aligned}
\thetahv_T:=&\argmin_{\thetahv}\left|1-\E(\|\ssv(\thetahv)\|_2^2)\right|\\
=&\argmin_{\thetahv}|\sigma^2_\varepsilon-V_{\X}(\thetahv)|.
\end{aligned}
\end{equation}
\end{prop}

Even if Proposition \ref{thm:mar_pas2} relates the selected model to the observation noise, an explicit value for $\thetahv_T$ is generally inaccessible. However, for linear ridge regression in the asymptotic setting, we can obtain a closed-form expression for the regularization $\lambda$, as is done in Theorem \ref{thm:mar_pas2}. We also show that the out-of-sample risk when selecting $\lambda$ based on MSV-Tr is bounded in terms of the theoretically optimal (inaccessible) risk. To distinguish the asymptotic quantities from the non-asymptotic, we denote them with a superscript bar.
\begin{thm}
\label{thm:mar_pas2}
For linear ridge regression, where $\ssv(\lambda):=\xsvt\left(\X^\top\X+n\lambda\I_d\right)^{-1}\X^\top$, if $n,d\to\infty$, such that $d/n\to\gamma\in(0,\infty)$, and the distribution of $\xv$ has zero mean, unit variance and a finite moment of order $4+\eta$ for some $\eta>0$, then
\begin{equation*}
\begin{aligned}
\overline{\lambda_T}:=&\argmin_{\lambda\geq 0}\left|\sigma^2_\varepsilon-\overline{V_{\X}}(\lambda)\right|\\
=&\begin{cases}
3\sqrt{\frac\gamma 2}-\gamma-1, & \gamma \in \left(\frac12,2\right)\\
0, & \gamma \in \left(0,\frac12\right]\cup\left[2,\infty\right),
\end{cases}
\end{aligned}
\end{equation*}
where $\overline{V_{\X}}$ denotes the variance component of the asymptotic conditional out-of-sample risk.

Additionally, for $\snr\in[1,80]$ and $\gamma\in(0,\infty)$,
\begin{equation}
\label{eq:mp3}
\asymrisk(\overline{\lambda_T})/\asymrisk(\overline{\lambda^*})<2.45,
\end{equation}
where $\snr:=\|\bm{\beta^*}\|_2^2/\sigma_{\varepsilon}^2$ (for $\bm{f}=\X\bm{\beta^*}$) denotes the signal-to-noise ratio and $\asymrisk(\lambda)$ denotes the asymptotic risk, which is minimized for $\overline{\lambda^*}=\gamma/\snr$.
\end{thm}
\begin{remark}
The assumption of zero mean and unit variance is quite weak, since this can always be obtained by properly rescaling and rotating the data.
\end{remark}

\begin{figure}[t]
\center
\includegraphics[width=\columnwidth]{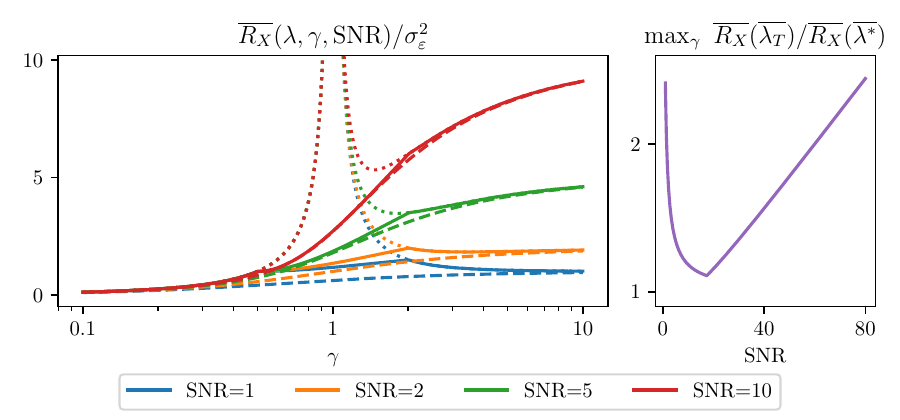}
\caption{Left: Asymptotic risks, $\asymrisk$, for different values of $\gamma=\lim_{n,d\to\infty}\frac dn$ and the signal-to-noise ratio, SNR. For $\ssmmrisk$ we use solid, for $\optrisk$ dashed, and for $\zerorisk$ dotted lines. 
For $\gamma\notin \left(\frac12,2\right)$, $\ssmmrisk$ coincides with $\zerorisk$, but for $\gamma\in \left(\frac12,2\right)$, where $\zerorisk$ tends to deviate substantially form $\optrisk$ (and even diverge for $\gamma=1$), $\ssmmrisk$ stays much closer to $\optrisk$. Right: $\ssmmrisk/\optrisk$ for SNR in $[1,80]$. In this interval, the quotient is always less than 2.45 (and, for some values of SNR, close to 1), which is expected from Theorem \ref{thm:mar_pas2}.}
\label{fig:mar_pas}
\vspace{-.3cm}
\end{figure}

In the left panel of Figure \ref{fig:mar_pas}, we plot $\ssmmrisk$ for different values of $\gamma$ and SNR together with $\optrisk$ and $\asymrisk(\lambda=0)$. $\ssmmrisk$ follows $\optrisk$ quite closely, avoiding the divergence at $\gamma=1$ that occurs for $\zerorisk$. In the right panel, we plot the maximum ratio of the MSV-Tr to the optimal asymptotic risk as a function of the SNR. For SNR $\in[1,80]$, the quotient is less than 2.45, which is in accordance with Equation \ref{eq:mp3}. It is, however, easy to verify that the quotient diverges both for SNR$\to 0$ and SNR$\to\infty$, which suggests that MSV-Tr does not work well for extreme SNRs, unless modified by introducing the $a$ discussed above.

\subsection{What About \texorpdfstring{$\yv$}{y}-free Cross Validation?}
Cross-validation has no closed-form in general, and can thus not be made independent of $\yv$ in the same way as MSV. However, leave-one-out cross-validation, LOOCV \citep{allen1974relationship}, and its close relative generalized cross-validation, GCV \citep{golub1979generalized}, have closed-form solutions, where the latter amounts to minimizing
\begin{equation*}
\frac1n\sum_{i=1}^n\left(\frac{y_i-\fh_i}{1-\text{Tr}(\Ss)/n}\right)^2
=n\cdot\frac{\yv^\top(\I_n-\Ss)^\top(\I_n-\Ss)\yv}{\left(\Tr(\I_n-\Ss)\right)^2},
\end{equation*}
which can be made independent of $\yv$ in the same ways as MSV.\footnote{Since LOOCV has a more complicated form than GCV, it is less amenable to theoretical analysis. However, given the similarities between the methods, it should behave very similarly to GCV. This is also our experience from empirical experiments.} 

However, in Theorem \ref{thm:opt_lbda}, we show that, in contrast to MSV, $\yv$-free, norm-based GCV cannot be used for selecting the regularization strength $\lambda$ for ridge regression \emph{regardless} of the feature expansion. The theorem thus includes, e.g., linear and kernel ridge regression.
\begin{thm}~\\
\label{thm:opt_lbda}
Let $\Ss_\lambda:=\Ph\Ph^+_\lambda$ and $\ssv_\lambda:=\phsvt\Ph^+_\lambda$, where 
$$\Ph^+_\lambda:=(\Ph^\top\Ph+\lambda\I_p)^{-1}\Ph^\top=\Ph^\top(\Ph\Ph^\top+\lambda\I_n)^{-1},$$
and let $\|\cdot\|$ denote the trace seminorm, the nuclear norm, or the Frobenius norm. Then
\begin{enumerate}[label=(\alph*)]
\item
for GCV,
$\argmin_{\lambda\geq 0}\frac{\left\|(\I_n-\Ss_\lambda)^\top(\I_n-\Ss_\lambda)\right\|}{\left(\Tr(\I_n-\Ss_\lambda)\right)^2}=\infty$.
\item
for in-sample MSV,\\
$\argmin_{\lambda\geq 0}\left\|\frac1n\I_n-\frac1n\Ss^\top_\lambda\Ss_\lambda\right\|=$\\ $\argmin_{\lambda\geq 0}\left\|\frac1n\I_n-\frac1n\Ss_\lambda\right\|=0$.
\item
for out-of-sample MSV, if $\E(\phsv\phsvt)=\I_p$, then\\
$\argmin_{\lambda\geq 0}\left\|\frac1n\I_n-\E_{\xv}(\ssv_\lambda\ssvt_\lambda)\right\|<\infty$.\\
If, in addition, $\|\Ph\Ph^\top\|_2<n$, then\\
$\argmin_{\lambda\geq 0}\left\|\frac1n\I_n-\E_{\xv}(\ssv_\lambda\ssvt_\lambda)\right\|>0$.
\end{enumerate}
\end{thm}
\begin{remark}
Thanks to its close connection to ridge regression, it is not difficult to extend Theorem \ref{thm:opt_lbda} to gradient flow, i.e., gradient descent with infinitesimal learning rate, as is done in Appendix \ref{sec:opt_lbda2}. Through this extension, we also include, e.g., neural networks via the NTK framework.
\end{remark}
\begin{remark}
For the spectral norm, the analog of Theorem \ref{thm:opt_lbda} becomes slightly more complicated, but the conclusions are the same. See Appendix \ref{sec:opt_lbda2} for details.
\end{remark}

For the two in-sample methods, GCV and in-sample MSV, $\yv$-free model selection always selects either infinite (for GCV) or zero (for in-sample MSV) regularization, regardless of the data and the feature expansion, and thus cannot be used for parameter selection. (Note that a training time of 0 corresponds to an infinitely regularized model.)
For out-of-sample MSV, things become more interesting. Assuming that the features are isotropic, which, for linear features, can be obtained by properly rescaling and rotating the data, we see that the optimal regularization is always less than infinity, and larger than zero if $\|\Ph\Ph^\top\|_2< n$, which is consistent with Proposition~\ref{thm:mar_pas1}.

\section{\texorpdfstring{$\yv$}{y}-FREE TRAINING OF NEURAL NETWORKS, NEAREST NEIGHBORS, AND RANDOM FORESTS}
\label{sec:snn}
To apply our developments to neural networks, k-nearest neighbors (kNN), and random forests (RF), these models must first be formulated as linear smoothers. In this section, we show how to do that.
\subsection{Neural Networks}
\citet{jeffares2024deep} expressed neural networks as nonlinear smoothers, based on the NTK framework. Our formulation extends their work with two important distinctions: First, in addition to the squared loss, our formulation allows for using the cross-entropy loss, which we use for classification. Second, to obtain a linear smoother, we train the neural network on \emph{non-informative labels/response}, not adding the true $\yv$ until after all network parameters have been fixed. The details are given in Appendix \ref{sec:snn1}.

Calculating the smoother matrix requires calculating the NTK in every iteration when training the network, which can be expensive both in terms of computation and storage. The purpose of formulating a neural network as a linear smoother is, however, not to obtain more efficient training, but to show that it is possible to successfully train the network without using the information in~$\yv$. 
Another limitation of the smoother formulation is that, for neural networks, the smoother matrix cannot be obtained in closed form, but it is instead computed iteratively during training. Thus, the $\xsv$-data of all desired future predictions need to be included during training.

\subsection{k-Nearest Neighbors and Random Forests}
In the smoother formulations of k-nearest neighbors and decision trees, $\ssv(\xsv,\xvi)$ is non-zero only if $\xvi$ belongs to the neighborhood of $\xsv$ (for k-nearest neighbors), or is in the same leaf as $\xsv$ (for decision trees).
For random forests, the total smoother is an aggregation of the smoothers of the individual trees. The details are given in Appendix \ref{sec:knn_rf}.

\section{EXPERIMENTS}
\label{sec:exps}

In this section, we investigate $\yv$-free training for a smoothing spline, linear ridge regression (LRR), kernel ridge regression (KRR), neural network regression (NNR), neural network classification (NNC), k-nearest neighbor regression (kNNR), k-nearest neighbor classification (kNNC), random forest regression (RFR), and random forest classification (RFC) on one synthetic and six real-world data sets, which are presented in Table \ref{tab:datasets}.

The exact experimental setups are given in Appendix~\ref{sec:exp_dets}; here follows a summary: For KRR, we used the Gaussian kernel. For NNR and NNC, we used a one-hidden-layer feed-forward neural network, trained with gradient descent with momentum and the squared (for NNR) or cross-entropy (for NNC) loss. 

For the real data sets, we randomly sampled 500 training and 100 test observations. For MSV, we sampled 500 validation observations (only covariates, $\Xs$) from a multivariate normal distribution parameterized by the sample mean and covariance of the training data. This was repeated 10 times for different subsets of the data. 

To select regularization/stopping time/number of neighbors, we compare 10-fold $\yv$-based cross-validation to Frobenius norm-based MSV. For NNR and NNC on the real data, $\yv$-based training was instead obtained by using $\yv$ during training and setting aside 20\% of the training data for validating when to stop training (henceforth referred to as the standard method). To train the neural networks and random forests without $\yv$, we used a random $\yvr\in\R^n$, where, for NNR and RFR, $\yvr\sim\N(\nv,\I_n)$, and, for NNC and RFC, $(\yvr)_i\sim \text{Cat}(c)$, where $\N(\bm{\mu},\bm{\Sigma})$ denotes the normal distribution and Cat$(c)$ the categorical distribution for $c$ categories of equal probability. For the neural networks, we monitored the Frobenius norm-based MSV and GCV to decide when to stop training. For the random forests, we used the default parameters of the scikit-learn implementation.

\begin{figure}[t]
\center
\includegraphics[width=\columnwidth]{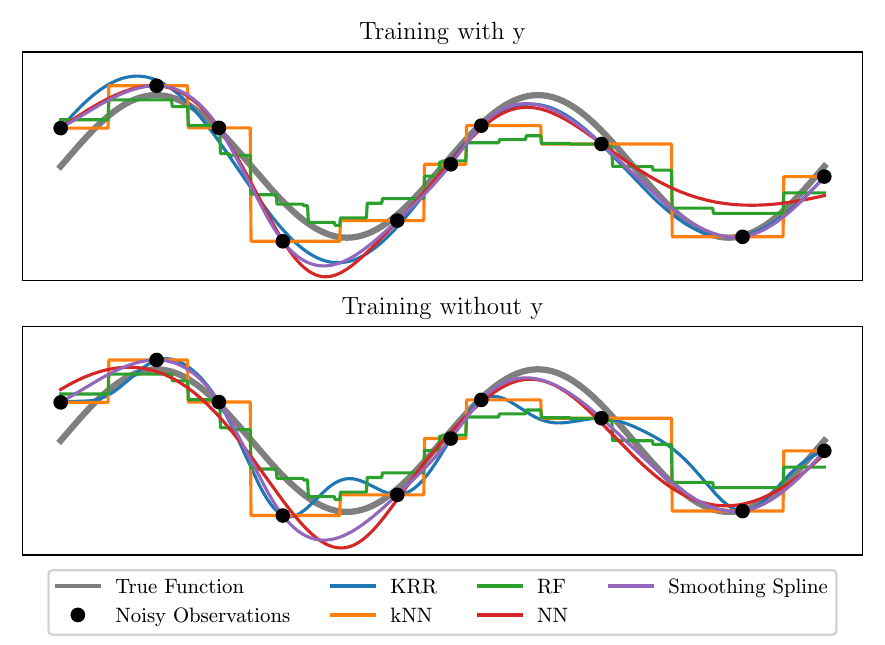}
\caption{The results of training different models with and without $\yv$ on synthetic data. The models trained without $\yv$ perform similarly to those trained with $\yv$. In the bottom panel, all five models are expressed in the form $\fhs=\ssvt\yv$, where $\ssv$ is constructed independently of $\yv$.}
\label{fig:wo_y_demo}
\vspace{-.3cm}
\end{figure}

\begin{table}
\caption{Real data sets, with number of features, $d$.}
\center
\begin{tabular}{@{}ll@{}}
\toprule
Data Set & $d$\\
\midrule
\makecell[l]{MNIST \citep{lecun1998gradient}} & 784\\
\makecell[l]{CIFAR-10 \citep{krizhevsky2009learning}} & 3072\\
\makecell[l]{Energy consumption in steel\\production \citep{ve2021efficient}} & 6\\
\makecell[l]{Run time of CPUs \citep{delve1996comp}} & 21\\
\makecell[l]{Critical temperature\\ of superconductors \citep{hamidieh2018data}} & 81\\
\makecell[l]{Power consumption of Tetouan\\City \citep{salam2018comparison}} & 7\\
\bottomrule
\end{tabular}
\label{tab:datasets}
\vspace{-.3cm}
\end{table}

The results are displayed in Figure \ref{fig:wo_y_demo} and Table \ref{tab:res_real}. In the table, for each method, we present the median (Q2) and first (Q1) and third (Q3) quartiles, over the 10 realizations, of the accuracy (for classification) or $R^2$ (for regression) \footnote{$R^2:=1-\|\yv-\fhv\|_2^2/\|\yv-\bar{y}\|_2^2\leq 1$, where $\bar{y}$ denotes the mean of $\yv$, measures the proportion of the variance in $\yv$ that is explained by the model, and is a normalized version of the mean squared error. $R^2=1$ corresponds to a perfect fit, while a model that predicts $\fh_i=\bar{y}$ for all $i$ results in $R^2=0$; thus a negative value of $R^2$ is possible, but it corresponds to a model that performs worse than always predicting the mean of the observed data.} on the test data. 
In all cases, $\yv$-free MSV performs much better than random guessing, most of the time at the level of the $\yv$-based methods, and sometimes even slightly better. $\yv$-free GCV performs at the level of random guessing, which corresponds to an accuracy of 0.1 for classification, since both MNIST and CIFAR-10 have ten different classes, and to an $R^2$ of 0 for regression. This is expected from Theorem \ref{thm:opt_lbda}, according to which, $\yv$-free GCV-based training uses infinite regularization (for neural networks in the form of zero training time).\footnote{Technically, Theorem \ref{thm:opt_lbda} does not apply to kNN. However, $k$ is a sort of regularization, since a larger $k$ results in a more stable model. For kNN, $\yv$-free GCV selects $k=n$, which is analog to $\lambda=\infty$.}\textsuperscript{,}\footnote{$\yv$-free LOOCV performs identically to $\yv$-free GCV on all methods and data, and is thus not reported separately.}

The larger gap between $\yv$-based and $\yv$-free training of random forests compared to the other models can be attributed to the importance of $\yv$ when selecting the splits in the decision trees.\footnote{We also investigated using a vector of only zeros, rather than $\yvr$, for $\yv$-free training. For neural networks, the results were virtually identical, but for RF the results of using $\nv$ were at the level of random guessing, which is not surprising: If all y values are identical, the decision tree consists of a single leaf containing all training data, and the regression model simply predicts the mean, while the classification model always predicts the same class.}
The limited performance of NNC on MNIST and CIFAR-10, compared to state-of-the-art, can be attributed to the fact that we only use 500 training observations, due to the high computational cost of calculating the smoother matrix for neural networks.

\begin{table*}[!ht]
\caption[hej]{Median and first and third quartile, over the 10 splits, of the accuracy (for classification) or $R^2$ (for regression) on the test data. The models trained with and without $\yv$ perform similarly. $\yv$-free GCV performs at the level of random guessing.}
\centering
\begin{minipage}[t]{0.49\textwidth}
\vspace{0pt}
\begin{tabular}{@{}l l l l@{}}
\toprule
Data & Model & Method & \makecell[l]{Accuracy/$R^2$\\ Q2 (Q1, Q3)} \\
\midrule
\multirow{8}{*}{MNIST}
 & \multirow{3}{*}{NNC}  & Standard & 0.82 (0.78, 0.85)  \\
 &                       & MSV     & 0.75 (0.74, 0.79)  \\
 &                       & GCV      & 0.11 (0.10, 0.12)  \\
\cline{2-4}
 & \multirow{3}{*}{kNNC} & Standard & 0.78 (0.76, 0.79) \\
 &                       & MSV     & 0.76 (0.73, 0.79) \\
 &                       & GCV      & 0.09 (0.08, 0.13) \\
\cline{2-4}
 & \multirow{2}{*}{RFC}  & Standard & 0.83 (0.79, 0.85) \\
 &                       & Random $\yv$ & 0.29 (0.25, 0.34) \\
\hline
\multirow{8}{*}{CIFAR-10}
 & \multirow{3}{*}{NNC}  & Standard & 0.32 (0.30, 0.34)  \\
 &                       & MSV     & 0.33 (0.29, 0.38)  \\
 &                       & GCV      & 0.09 (0.07, 0.11)  \\
\cline{2-4}
 & \multirow{3}{*}{kNNC} & Standard & 0.23 (0.19, 0.25) \\
 &                       & MSV     & 0.25 (0.18, 0.26) \\
 &                       & GCV      & 0.12 (0.10, 0.14) \\
\cline{2-4}
 & \multirow{2}{*}{RFC}  & Standard & 0.33 (0.31, 0.37) \\
 &                       & Random $\yv$ & 0.21 (0.16, 0.25) \\
\hline
\multirow{14}{*}{\makecell[l]{Steel\\Produc-\\tion}}
 & \multirow{3}{*}{LRR}  & CV   & 0.98 (0.98, 0.99)   \\
 &                       & MSV & 0.99 (0.98, 0.99)   \\
 &                       & GCV  & 0.00 (-0.02, 0.00)   \\
\cline{2-4}
 & \multirow{3}{*}{KRR}  & CV   & 1.00 (1.00, 1.00)   \\
 &                       & MSV & 0.95 (0.92, 0.96)   \\
 &                       & GCV  & 0.00 (-0.02, 0.00)   \\
\cline{2-4}
 & \multirow{3}{*}{NNR}  & Standard & 0.95 (0.95, 0.96)  \\
 &                       & MSV     & 0.98 (0.97, 0.98)  \\
 &                       & GCV      & 0.00 (-0.02, 0.00)   \\
\cline{2-4}
 & \multirow{3}{*}{kNNR} & CV   & 0.98 (0.98, 0.99)  \\
 &                       & MSV & 0.98 (0.98, 0.99)  \\
 &                       & GCV  & 0.00 (-0.03, 0.00)   \\
\cline{2-4}
 & \multirow{2}{*}{RFR}  & Standard     & 0.99 (0.99, 0.99)  \\
 &                       & Random $\yv$ & 0.82 (0.79, 0.85)  \\
\hline
\multirow{6}{*}{\makecell[l]{CPU\\Run\\Time}}
 & \multirow{3}{*}{LRR}  & CV   & 0.68 (0.63, 0.72)   \\
 &                       & MSV & 0.71 (0.67, 0.74)   \\
 &                       & GCV  & 0.00 (0.00, 0.00)   \\
\cline{2-4}
 & \multirow{3}{*}{KRR}  & CV   & 0.91 (0.89, 0.93)   \\
 &                       & MSV & 0.65 (0.51, 0.76)   \\
 &                       & GCV  & 0.00 (0.00, 0.00)   \\
\bottomrule
\end{tabular}
\end{minipage}\hfill
\begin{minipage}[t]{0.49\textwidth}
\vspace{0pt}
\begin{tabular}{@{}l l l l@{}}
\toprule
Data & Model & Method & \makecell[l]{Accuracy/$R^2$\\ Q2 (Q1, Q3)} \\
\midrule
\multirow{8}{*}{\makecell[l]{CPU\\Run\\Time}}
 & \multirow{3}{*}{NNR}  & Standard & 0.69 (0.63, 0.72)  \\
 &                       & MSV     & 0.71 (0.64, 0.80)  \\
 &                       & GCV      & 0.00 (0.00, 0.00)  \\
\cline{2-4}
 & \multirow{3}{*}{kNNR} & CV   & 0.85 (0.65, 0.93) \\
 &                       & MSV & 0.83 (0.68, 0.92) \\
 &                       & GCV  & 0.00 (0.00, 0.00) \\
\cline{2-4}
 & \multirow{2}{*}{RFR}  & Standard     & 0.96 (0.95, 0.98) \\
 &                       & Random $\yv$ & 0.34 (0.25, 0.42) \\
\hline
\multirow{14}{*}{\makecell[l]{Super-\\conduc-\\tors}}
 & \multirow{3}{*}{LRR}  & CV   & 0.71 (0.65, 0.73)   \\
 &                       & MSV & 0.71 (0.65, 0.74)   \\
 &                       & GCV  & 0.01 (0.01, 0.01)   \\
\cline{2-4}
 & \multirow{3}{*}{KRR}  & CV   & 0.76 (0.71, 0.81)   \\
 &                       & MSV & 0.73 (0.70, 0.76)   \\
 &                       & GCV  & 0.00 (0.00, 0.00)   \\
\cline{2-4}
 & \multirow{3}{*}{NNR}  & Standard & 0.67 (0.65, 0.70)  \\
 &                       & MSV     & 0.72 (0.68, 0.73)  \\
 &                       & GCV      & 0.00 (0.00, 0.00)  \\
\cline{2-4}
 & \multirow{3}{*}{kNNR} & CV   & 0.73 (0.71, 0.74)  \\
 &                       & MSV & 0.74 (0.72, 0.76)  \\
 &                       & GCV  & 0.00 (0.00, 0.00)  \\
\cline{2-4}
 & \multirow{2}{*}{RFR}  & Standard     & 0.80 (0.77, 0.83)  \\
 &                       & Random $\yv$ & 0.51 (0.47, 0.53)  \\
\hline
\multirow{14}{*}{\makecell[l]{Power\\Consump-\\tion}}
 & \multirow{3}{*}{LRR}  & CV   & 0.61 (0.57, 0.68)   \\
 &                       & MSV & 0.61 (0.57, 0.68)   \\
 &                       & GCV  & 0.00 (0.00, 0.00)   \\
\cline{2-4}
 & \multirow{3}{*}{KRR}  & CV   & 0.79 (0.77, 0.81)  \\
 &                       & MSV & 0.75 (0.74, 0.79)  \\
 &                       & GCV  & 0.00 (0.00, 0.00)  \\
\cline{2-4}
 & \multirow{3}{*}{NNR}  & Standard & 0.69 (0.67, 0.73)  \\
 &                       & MSV     & 0.68 (0.59, 0.71)  \\
 &                       & GCV      & 0.00 (0.00, 0.00)  \\
\cline{2-4}
 & \multirow{3}{*}{kNNR} & CV   & 0.73 (0.69, 0.76)    \\
 &                       & MSV & 0.74 (0.71, 0.76)    \\
 &                       & GCV  & 0.00 (-0.01, 0.00)   \\
\cline{2-4}
 & \multirow{2}{*}{RFR}  & Standard     & 0.77 (0.76, 0.80)    \\
 &                       & Random $\yv$ & 0.52 (0.49, 0.54)    \\
\bottomrule
\end{tabular}
\end{minipage}
\label{tab:res_real}
\end{table*}

\section{CONCLUSIONS}
We demonstrated how standard supervised models, including neural networks, can be successfully trained in an unsupervised manner. We did this by formulating each model as a smoother, i.e.\ on the form $\fhs=\ssvt\yv$, and constructing the smoother, $\ssv$, without using the information in $\yv$. To construct the smoother independently of $\yv$, we introduced a model selection criterion, MSV, that is based on the out-of-sample predictions of the model, but independent of $\yv$. For iteratively trained models, including neural networks, we replaced the true outputs, $\yv$, with either a random vector or the zero vector during training.

We showed how MSV can be expressed in terms of the variance component of the out-of-sample risk, and, for linear ridge regression with $n,d\to\infty$, bounded the asymptotic risk of MSV in terms of the optimal asymptotic risk.

Using both real and synthetic data, we showed how five different linear and nonlinear models for classification and regression, trained without $\yv$, perform on par with the standard, $\yv$-based versions, and much better than random guessing. This suggests that, under the hood, supervised and unsupervised learning are closely related, and also supports the hypothesis that, given $\X$, $\xsv$, and $\yv$, the optimal model is a weighted sum of $\yv$, with weights that depend on $\X$ and $\xsv$ only.

We do not suggest that $\yv$ should be excluded from the training process in practice, partly because the smoother formulation scales poorly with the data. This holds especially true for neural networks, where the smoother matrix depends on the neural tangent kernel, which needs to be repeatedly recalculated during training.

While our model selection criterion performs well in practice, we have no evidence that it is optimal, leaving room for future improvements in $\yv$-free training.

\section*{Acknowledgements}
This research was supported by the \emph{Wallenberg AI, Autonomous Systems and Software Program (WASP)} funded by Knut and Alice Wallenberg Foundation, and by the \emph{Kjell och M{\"a}rta Beijer Foundation}.

\newpage
~\\
\newpage
~\\
\newpage
\bibliography{refs}
\bibliographystyle{apalike}

\newpage
\appendix
\onecolumn
In Appendices \ref{sec:exp_dets}, \ref{sec:snn1}, and \ref{sec:snn2}, we use as superscript $^\star$ to denote the concatenation of in-sample (training) and out-of-sample quantities, i.e., $\fhv^\star:=\begin{bmatrix}\fhv\\\fhv^*\end{bmatrix}$, $\X^\star:=\begin{bmatrix}\X\\\X^*\end{bmatrix}$, and $\Ss^\star:=\begin{bmatrix}\Ss\\\Ss^*\end{bmatrix}$.

\section{EXPERIMENTAL DETAILS}
\label{sec:exp_dets}
\subsection{Details for Figures \ref{fig:compl_demo} and  \ref{fig:wo_y_demo}}
The true function is given by $y=\sin(2\pi x)$. We sampled $n=10$ training observations according to $x_i\sim \mathcal{U}(-1,1)$, $y_i\sim\mathcal{N}(\sin(2\pi x_i), 0.3^2)$, where $\mathcal{U}(a,b)$ denotes the uniform distribution on $[a,b]$, and $\mathcal{N}(\mu,\sigma^2)$ denotes the normal distribution with mean $\mu$ and variance $\sigma^2$. In addition, we used $n^*=1000$ uniformly spaced out-of-sample data points on the interval [-1,1]. 

For kernel ridge regression, we used the Gaussian kernel, $k(\xv,\xpv,\sigma)=\exp\left(-\frac{\|\xv-\xpv\|_2^2}{2\sigma^2}\right)$. The predictions are given in closed form by $\fhv^\star=\K^\star\left(\K+\lambda\I_n\right)^{-1}\yv$, where $\K^\star=[\K^\top,\K^{*\top}]^\top\in\R^{(n+n^*)\times n}$, and $\K^\star_{ij}= k(\bm{x_i}^\star,\bm{x_j},\sigma)$.

For Figure \ref{fig:compl_demo} we used $\lambda=0$ and $\sigma\in\{1.3,0.16,0.01\}$.

For the smoothing spline, we used B-splines of degree 3 to create the basis matrix $\B^\star=[\B^\top, \B^{*\top}]^\top\in\R^{(n+n^*)\times (n+4)}$, and the penalty matrix $\bm{\Omega}\in\R^{(n+4)\times (n+4)}$, see e.g. Chapter 5 in \citet{hastie2009elements} for details. The predictions are given in closed form by $\fhv^\star=\B^\star\left(\B^\top\B+\lambda\bm{\Omega}\right)^{-1}\B^\top\yv$.

For the neural network, we used one hidden layer with 20 nodes and the tanh activation function, trained with gradient descent with momentum, with learning rate $\eta=0.01$, momentum $\gamma=0.95$, and the squared loss.

We used 100 logarithmically spaced candidate values between $10^{-4}$ and 1 for $\lambda$ and $\sigma$ respectively, and $k\in[1,10]$ for k-nearest neighbors.
For $\yv$-based training, we selected $\lambda$, $\sigma$, and $k$ via leave-one-out cross-validation, and the stopping epoch via a 80/20\% split of the data into training/validation sets.
For $\yv$-free training, we used Frobenius norm-based MSV and trained the neural network and the random forest on $\yvr\in\R^n$, where $\yvr\sim\N(\nv,\I_n)$.

\subsection{Details for the Real Data in Section \ref{sec:exps}}
For KRR we used the Gaussian kernel, $k(\xv,\xpv,\sigma)=\exp\left(-\frac{\|\xv-\xpv\|_2^2}{2\sigma^2}\right)$. For NNR and NNC, we used one hidden layer with 200 nodes and the tanh activation function, trained with gradient descent with momentum, with learning rate $\eta=10^{-4}$, and momentum $\gamma=0.7$. For NNR, we used the squared loss, and for NNC, the cross-entropy loss. For each data set, the covariates, $\X^\star$, were standardized to zero mean and unit variance. For regression, the response, $\yv$, was standardized to zero mean, and for classification, the labels were one-hot encoded. We randomly sampled 500 training and 100 test observations. For MSV, we sampled 500 validation observations (only covariates, $\Xs$) from a multivariate normal distribution parameterized by the sample mean and covariance of the training data. This was repeated 10 times for different subsets of the data.

For model selection, we compare generalized (for the synthetic data) or 10-fold (for the real data) $\yv$-based cross-validation to Frobenius norm-based MSV. For NNR and NNC on the real data, $\yv$-based training was instead obtained by using $\yv$ during training and setting aside 20\% of the training data for validating when to stop training (henceforth referred to as the standard method). To train the neural networks and random forests without $\yv$, we used a random $\yvr\in\R^n$, where, for NNR and RFR, $\yvr\sim\N(\nv,\I_n)$, and, for NNC and RFC, $(\yvr)_i\sim \text{Cat}(k)$, where $\N(\bm{\mu},\bm{\Sigma})$ denotes the normal distribution and Cat$(k)$ the categorical distribution for $k$ categories of equal probability. For the neural networks, we monitored the Frobenius norm-based MSV and GCV to decide when to stop training. For the random forests, we used the default parameters of the scikit-learn implementation.

For LRR, KRR, and kNN, we compared three different model selection methods: 10-fold, $\yv$-based cross-validation, and $\yv$-free, Frobenius norm-based MSV and GCV. For LRR and KRR, we evaluated 200 logarithmically spaced candidate values between $10^{-4}$ and 20, plus $10^6$, for $\lambda$ and $\sigma$, respectively. For kNN, we investigated $k$ between 2 and 30, plus $k=n$ (for 10-fold cross-validation the maximum $k$ was $9/10\cdot n$).
For $\yv$-based training of the neural network models, we used $\yv$ and set aside 20\% of the training data for validating when to stop training (henceforth referred to as the standard method). To train the neural networks and random forests without $\yv$, we used a random $\yvr\in\R^n$, where, for NNR and RFR, $\yvr\sim\N(\nv,\I_n)$, and, for NNC and RFC, $(\yvr)_i\sim \text{Cat}(c)$, where $\N(\bm{\mu},\bm{\Sigma})$ denotes the normal distribution and Cat$(c)$ the categorical distribution for $c$ categories of equal probability. For NNC, $\yvr$, was transformed into a $n\times(c-1)$ one-hot matrix (see Appendix \ref{sec:cross_entr} for details). For the neural networks, we monitored the Frobenius norm-based MSV and GCV to decide when to stop training. For the random forests, we used the default parameters of Python's scikit-learn implementation. 

Each experiment took between 0.1 seconds for LRR and 45 minutes for NNC on the CIFAR-10 data to run on an Intel i9-13980HX processor with access to 32 GB RAM. The total run time, for all experiments, was approximately 10 hours.

The MNIST data are available at \url{http://yann.lecun.com/exdb/mnist/} under the CC-BY-SA 3.0 license.\\
The CIFAR-10 data are available at \url{https://www.cs.toronto.edu/~kriz/cifar.html} under the MIT license.\\
The Steel Production data are available at\\\url{https://archive.ics.uci.edu/dataset/851/steel+industry+energy+consumption}\\under the CC-BY 4.0 license. We used the features \texttt{Lagging\_Current\_Reactive.Power\_kVarh}, \texttt{Leading\_Current\_Reactive\_Power\_kVarh}, \texttt{CO2(tCO2)}, \texttt{Lagging\_Current\_Power\_Factor}, \texttt{Leading\_Current\_Power\_Factor}, and \texttt{NSM} to predict \texttt{Usage\_kWh}.\\
The CPU Run Time data are available at\\\url{http://www.cs.toronto.edu/~delve/data/comp-activ/desc.html}. We excluded the features \texttt{usr}, \texttt{sys}, \texttt{wio}, and \texttt{idle}.\\
The Superconductor data are available at\\\url{https://archive.ics.uci.edu/dataset/464/superconductivty+data}\\under the CC-BY 4.0 license.\\
The Power Consumption data are available at\\\url{ https://archive.ics.uci.edu/dataset/849/power+consumption+of+tetouan+city}\\under the CC-BY 4.0 license. We excluded the \texttt{DateTime} feature, using the remaining features to predict \texttt{Zone 3 Power Consumption}.

\section{\texorpdfstring{$\yv$}{y}-FREE TRAINING OF NEURAL NETWORKS IN MORE DETAIL}
\label{sec:snn1}
In Theorem \ref{thm:snn_k}, we show how we can construct a neural network smoother matrix iteratively during training. We thus obtain a model in the form $\Ss^\star(\thetahv)\cdot\yv$, where $\thetahv\in\R^p$ denotes the model parameters, that approximates $\fhv^\star(\thetahv)$ to arbitrary precision, given a small enough learning rate. 

Since $\thetahv$ depends on $\yv$, so does $\Ss^\star$. To make $\Ss^\star$ independent of $\yv$, we can simply replace $\yv$ with a random, or zero, label/response vector during training, i.e.\ instead of 
$\fhv^\star=\Ss^\star(\thetahv(\yv))\cdot\yv$, we use $\fhv^\star=\Ss^\star(\thetahv(\yvr))\cdot\yv$, where $\yvr$ is a vector sampled at random, or $\fhv^\star=\Ss^\star(\thetahv(\nv))\cdot\yv$.
To decide when to stop training, i.e.\ which $\Ss^\star(\thetahv_k)$ to use, we can monitor the $\yv$-free version of MSV during training.

\begin{thm}~\\
\label{thm:snn_k}
Let the neural network be trained with gradient descent with momentum, with learning rate $\eta>0$, and momentum $\gamma\in[0,1)$, with either the squared or the cross-entropy loss.\\
Let $\Ss_k^\star$ be updated according to
\begin{equation*}
\label{eq:snnss_k}
\begin{aligned}
&\Ss^\star_{-1}=\Ss^\star_{0}=\nv,\\ 
&\Ss^\star_{k+1}=\Ss^\star_k +\gamma\cdot\left(\Ss^\star_k-\Ss^\star_{k-1}\right)
+\eta\cdot\bm{\tilde{K}}^\star_{k+1}\cdot\left(\I_n-\Ss_{k}\right),\\
&k=0,1,\dots
\end{aligned}
\end{equation*}
where $\bm{\tilde{K}}^\star_k$
denotes a generalized time-dependent NTK (see Remark \ref{rem:ntk}).\\
Then, if all derivatives are bounded, there exists a constant, $C<\infty$, such that
\begin{equation}
\label{eq:ss_diff}
\left\|\fhv^\star(\thetahv_k)-\left(\Ss^\star_k\left(\yv-\fhv(\thetahv_0)\right)+\fhv^\star(\thetahv_0)\right)\right\|_\infty\leq \eta \cdot C.
\end{equation}
\end{thm}
\begin{remark}
If $\fhv^\star(\thetahv_0)=\nv$, Equation \ref{eq:ss_diff} simplifies to
$\left\|\fhv^\star(\thetahv_k)-\Ss^\star_k\yv\right\|_\infty\leq \eta \cdot C$.
\end{remark}
\begin{remark}
\label{rem:ntk}
For the square loss, $\bm{\tilde{K}}^\star_k=\K^{\star\text{\normalfont NTK}}_k=\left(\partial_{\thetahv_k}\fhv^\star(\thetahv_k)\right)\left(\partial_{\thetahv_k}\fhv(\thetahv_k)\right)^\top$ is the standard time-dependent NTK.
For the cross-entropy loss, $\bm{\tilde{K}}^\star_k=\K^{\star\text{\normalfont NTK}}_k\cdot\bm{\tilde{F}}_k$, where $\bm{\tilde{F}}_k$ is a square, block-diagonal matrix that depends on $\fhv(\thetahv_k)$. The details are given in Appendix \ref{sec:snn2}.
\end{remark}
\begin{remark}
Since $\bm{\tilde{K}}^\star_k=\bm{\tilde{K}}^\star(\thetahv_k)$, $\Ss^\star_k=\Ss^\star\left(\{\bm{\tilde{K}}^\star_{l}\}_{l=0}^k\right)=\Ss^\star\left(\{\thetahv_{l}\}_{l=0}^k\right)$, i.e.\ $\Ss^\star_k$ depends on all historical values of $\thetahv$ during training.
\end{remark}
\begin{remark}
When $\bm{\tilde{K}}^\star$ is constant during training, the smoother formulation is exact, i.e.\ $C=0$ in Equation \ref{eq:ss_diff}. This is the case, for instance, when solving linear or kernel regression with gradient descent.
\end{remark}

\section{\texorpdfstring{$\yv$}{y}-FREE TRAINING OF k-NEAREST NEIGHBORS AND RANDOM FORESTS}
\label{sec:knn_rf}
We formulate k-nearest neighbors (kNN) and random forests (RF) as linear smoothers for regression, in which case the predicted value is the average of the training data. For classification, the average is simply replaced by the mode (i.e.\ the most common value).

For k-nearest neighbors, each prediction is defined as the average of the $k$ closest training points:
\begin{equation*}
\fhs_{NN}(\xsv):=\frac1k \cdot \sum_{i:\xvi\in N_k(\xsv)}y_i,
\end{equation*}
where $N_k(\xvi)$ is the neighborhood of $\xsv$, consisting of the $k$ closest training points $\xvi$.

Defining
$$(\ssv_{NN}(\xsv,\X))_i=\ssv_{NN}(\xsv,\xvi):=
\begin{cases}
\frac1k\text{ if }\xvi\in N_k(\xsv),\\
0\text{ if }\xvi\notin N_k(\xsv),
\end{cases}
$$
we obtain 
$$\fhs_{NN}(\xsv)=\frac1k \cdot \sum_{i:\xvi\in N_k(\xsv)}y_i=\ssvt_{NN} \yv.$$

For decision trees, the input space is split into $M$ non-overlapping regions, where each prediction is defined as the average of the training points that belong to the same region as the point of interest:
\begin{equation*}
\fhs_T(\xsv):=\frac1{|R(\xsv)|} \cdot \sum_{i:\xvi\in R(\xsv)}y_i,
\end{equation*}
where $R(\xsv)$ is the region that $\xsv$ belongs to, and $|R(\xsv)|$ is the number of training points $\xvi$ in $R(\xsv)$.
Defining
$$(\ssv_{T}(\xsv,\X))_i=\ssv_{T}(\xsv,\xvi):=
\begin{cases}
\frac1{|R(\xsv)|}\text{ if }\xvi\in R(\xsv),\\
0\text{ if }\xvi\notin R(\xsv),
\end{cases}
$$
we obtain 
$$\fhs_T(\xsv)=\frac1{|R(\xsv)|} \cdot \sum_{i:\xvi\in R(\xsv)}y_i= \ssvt_T\yv.$$

The random forest prediction is the average of the predictions of $n_T$ trees, trained on different bootstrap samples of the training data:
$$\fhs_{RF}(\xsv):=\frac1{n_T}\cdot\sum_{t=1}^{n_T}\left(\frac1{|R_t(\xsv)|} \cdot \sum_{i:\xvi\in R_t(\xsv)}y_i\right)=\underbrace{\frac1{n_T}\cdot\sum_{t=1}^{n_T}\ssvt_{T,t}}_{=:\ssvt_{RF}}\yv=\ssvt_{RF}\yv.$$

\section{NEURAL NETWORKS AS SMOOTHERS IN EVEN MORE DETAIL}
\label{sec:snn2}
In this section, we describe in more detail how neural networks can be expressed as smoothers. First, we briefly review the neural tangent framework, which the smoother formulation builds upon. We then generalize the NTK, and thus the smoother, framework to hold also for the cross-entropy loss.

\subsection{Review of the Neural Tangent Kernel}
For univariate outputs, updating the parameters of a neural network (or actually any iteratively trained regression model) with loss function $L(\fhv,\yv)$, trained with gradient descent, amounts to
\begin{equation*}
\thetahv(t+\Delta t)=\thetahv(t)-\Delta t\cdot\frac{\partial L(\fhv(\thetahv(t)),\yv)}{\partial \thetahv(t)} =\thetahv(t)-\Delta t\cdot\left(\frac{\partial \fhv(\thetahv(t))}{\partial \thetahv(t)}\right)^\top\cdot\frac{\partial L(\fhv(\thetahv(t)),\yv)}{\partial \fhv(\thetahv(t))},
\end{equation*}
where we have used the chain rule, and where $\Delta t$ denotes the learning rate. Rearranging and letting $\Delta t\to 0$, we obtain the gradient flow update for $\thetahv$, 
\begin{equation*}
\frac{\partial \thetahv(t)}{\partial t}=-\left(\frac{\partial \fhv(\thetahv(t))}{\partial \thetahv(t)}\right)^\top\cdot\frac{\partial L(\fhv(\thetahv(t)),\yv)}{\partial \fhv(\thetahv(t))},
\end{equation*}
and with another application of the chain rule, the gradient flow update for $\fhv^\star$ is obtained as

\begin{equation}
\label{eq:ntk1}
\begin{aligned}
\frac{\partial \fhv^\star(\thetahv(t))}{\partial t}
&=\frac{\partial \fhv^\star(\thetahv(t))}{\partial \thetahv(t)}\cdot\frac{\partial \thetahv(t)}{\partial t}
=-\frac{\partial \fhv^\star(\thetahv(t))}{\partial \thetahv(t)}\cdot\left(\frac{\partial \fhv(\thetahv(t))}{\partial \thetahv(t)}\right)^\top\cdot\frac{\partial L(\fhv(\thetahv(t)),\yv)}{\partial \fhv(\thetahv(t))}\\
&=:-\Kntk(\thetahv(t))\cdot\frac{\partial L(\fhv(\thetahv(t)),\yv)}{\partial \fhv(\thetahv(t))},
\end{aligned}
\end{equation}
where
\begin{equation*}
\Kntk(t):= \frac{\partial \fhv^\star(t)}{\partial \thetahv}\cdot\left(\frac{\partial \fhv(t)}{\partial \thetahv}\right)^\top.
\end{equation*}

For the squared loss, $L(\fhv(t),\yv)=\frac12\left\|\yv-\fhv(t)\right\|_2^2$, $\frac{\partial L(\fhv(t),\yv)}{\partial \fhv}=\fhv(t)-\yv$, and 

\begin{equation}
\label{eq:ntk_sl}
\frac{\partial \fhv(t)}{\partial t}
=\Kntk(t)\cdot\left(\yv-\fhv(t)\right).
\end{equation}

\subsubsection{Multivariate Outputs}
\label{sec:multi_out}
For multivariate outputs, $\yv$ and $\fhv^\star$ are no longer vectors, but matrices, i.e.\ $\yv\in\R^n$, $\fhv^\star\in\R^{n+n^*}$, generalize to $\Y\in\R^{n\times d_y}$, $\Fh^\star\in \R^{(n+n^*)\times d_y}$, and Equation \ref{eq:ntk1} generalizes to
\begin{equation}
\label{eq:fh_gf4d}
\begin{aligned}
\left(\frac{\partial \hat{F}^\star(t)}{\partial t}\right)^{i_Ni_y}
&=-\left(\frac{\partial\hat{F}^\star(t)}{\partial \hat{\theta}}\right)^{i_Ni_y}_{i_p}\cdot\left(\frac{\partial\hat{F}(t)}{\partial \hat{\theta}}\right)^{i_p}_{j_nj_y}\cdot\left(\frac{\partial L(t)}{\partial \hat{F}}\right)^{j_nj_y}\\
&=:-K^{\star\text{\normalfont NTK}}(t)^{i_Ni_y}_{j_nj_y}\cdot\left(\frac{\partial L(t)}{\partial \hat{F}}\right)^{j_nj_y},
\end{aligned}
\end{equation}
where $\left(\frac{\partial \hat{F}^\star(t)}{\partial t}\right)^{i_Ni_y}\in\R^{(n+n^*)\times d_y}$, $\left(\frac{\partial \hat{F}^\star(t)}{\partial \thetah}\right)^{i_Ni_y}_{i_p}\in\R^{(n+n^*)\times d_y\times p}$, $\left(\frac{\partial\hat{F}(t)}{\partial \hat{\theta}}\right)^{i_ni_y}_{i_p}\in\R^{n\times d_y\times p}$, $\left(\frac{\partial L(t)}{\partial \hat{F}}\right)^{i_ni_y}\in\R^{n\times d_y}$, and $K^{\star\text{\normalfont NTK}}(t)^{i_Ni_y}_{j_nj_y}\in\R^{(n+n^*)\times d_y\times n\times d_y}$, and upper and lower indices are summed over according to Einstein's notation. However, by vectorizing $\Y$ and $\Fh^\star$ into vectors of length $n\cdot d_y$ and $(n+n^*)\cdot d_y$, i.e.\ $\Y=\left[\bm{y_1},\ \bm{y_2},\dots {\bm{y_n}}\right]^\top\in\R^{n\times d_y}\mapsto \yv=\left[\bm{y_1}^\top,\ \bm{y_2}^\top,\dots {\bm{y_n}}^\top\right]^\top\in\R^{nd_y}$, and analogously for $\Fh^\star$ and $\fhv^\star$, Equation \ref{eq:fh_gf4d} can be written on the same form as Equation \ref{eq:ntk1}, but for $\Kntk\in \R^{(n+n^*)d_y\times nd_y}$, rather than $\Kntk\in\R^{(n+n^*)\times n}$. To avoid a four-dimensional $\Kntk$, and thus a four-dimensional $\Ss^\star$, we use this formulation whenever dealing with multi-dimensional outputs.

\subsection{Classification with the Cross-Entropy Loss}
\label{sec:cross_entr}
For classification with more than two classes, we use compact one-hot encoding, where observation $i$ is represented as a vector $\yiv\in \R^{c-1}$, where $c$ is the number of classes. Note that it is enough with $c-1$ elements to represent $c$ classes: the first $c-1$ classes are encoded in the standard one-hot way (with one element in $\yiv$ being 1 and the remaining 0) while the last class is encoded as $\yiv=\nv$, i.e., it is none of the first $c-1$ classes. For $c=2$, this becomes the standard encoding for binary labels, with $y_i \in\{0,1\}$.

For the corresponding predictions, $\fhiv^\star\in\R^{c-1}$, we have $(\fhiv^\star)_j=:\fh^\star_{ij}\in[0,1]$, where the probability of the last class is obtained as $\fh^\star_{ic}=1-\sum_{j=1}^{c-1}\fh^\star_{ij}$. In this formulation, the cross-entropy loss is given by
\begin{equation}
\begin{aligned}
\label{eq:ce}
&L(\fhv,\yv)=\sum_{i=1}^n\Lt(\fhiv,\yiv):=\\
&-\sum_{i=1}^n\left(\sum_{j=1}^{c-1}y_{ij}\cdot\log(\fh_{ij})+\left(1-\sum_{j=1}^{c-1}y_{ij}\right)\cdot\log\left(1-\sum_{j=1}^{c-1}\fh_{ij}\right)\right),
\end{aligned}
\end{equation}
where $y_{ij}:=(\yiv)_j$ and 
\begin{equation*}
\Lt(\fhiv,\yiv)=-\left(\sum_{j=1}^{c-1}y_{ij}\cdot\log(\fh_{ij})+\left(1-\sum_{j=1}^{c-1}y_{ij}\right)\cdot\log\left(1-\sum_{j=1}^{c-1}\fh_{ij}\right)\right).
\end{equation*}
\begin{remark}
With $\fh_{ij}=\left(e^{g(\bm{x_{ij}})}\right)/\left(1+\sum_{k\neq j}e^{g(\bm{x_{ik}})}\right)$, for some linear or nonlinear function $g(\cdot)$, for $c=2$, this formulation becomes exactly logistic regression.
\end{remark}
In Proposition \ref{thm:cross_entr}, we show that there are matrices $\bm{\tilde{F}_i}$ such that $\frac{\partial \Lt(\fhiv,\yiv)}{\partial \fhiv}$ can be written on the form $\bm{\tilde{F}_i}\cdot(\fhiv-\yiv)$. 

\begin{prop}~\\
\label{thm:cross_entr}
With $\Lt(\fhiv,\yiv)$ defined according to Equation \ref{eq:ce},
\begin{equation*}
\begin{aligned}
\frac{\partial \Lt(\fhiv,\yiv)}{\partial \fhiv}=\left((\Fhi)^{-1}+\frac1{1-\sum_{j=1}^{c-1}\fh_{ij}}\cdot\bm{1}\bm{1}^\top\right)\cdot\left(\fhiv-\yiv\right)
=:\Fhi\cdot(\yiv-\fhiv),
\end{aligned}
\end{equation*}
where $(\Fhi)^{-1}\in\R^{(c-1)\times(c-1)}$ denotes the diagonal matrix with $\left((\Fhi)^{-1}\right)_{jj}=\frac1{\fh_{ij}}$, and $\bm{1}\in\R^{c-1}$ denotes a vector of only ones.
\end{prop}

Using the vectorized forms, $\yv:=\text{vec}(\Y)\in\R^{n\cdot(c-1)}$, where $\Y\in\R^{n\times(c-1)}$, and equivalently for $\fhv$, as described in Section \ref{sec:multi_out}, we thus obtain
\begin{equation}
\label{eq:F_cross_entr}
\frac{\partial L(\fhv,\yv)}{\partial \fhv} =\bm{\tilde{F}}\cdot(\yv-\fhv),
\end{equation}
where $\bm{\tilde{F}}=\diag(\bm{\tilde{F}_1},\ \bm{\tilde{F}_2},\dots {\bm{\tilde{F}_n}}),\ \R^{n(c-1)\times n(c-1)}$ is a block diagonal matrix, with the $\bm{\tilde{F}_i}$s along the diagonal.
Thus, by replacing $\Kntk(t)$ with $\Kntk(t)\cdot\bm{\tilde{F}}(t)$ in Equation \ref{eq:ntk_sl}, the smoother framework trivially extends to the cross-entropy loss.

\section{GENERALIZATIONS OF THEOREM \ref{thm:opt_lbda}}
\label{sec:opt_lbda2}
In this section, we generalize Theorem \ref{thm:opt_lbda} to include gradient flow and the spectral norm.

\subsection{Theorem \ref{thm:opt_lbda} for Gradient Flow}
Gradient flow, i.e., gradient descent with an infinitesimal learning rate, has the closed-form solution 
$$\bm{\hat{\beta}}_{\text{GF}}=(\I_d-\exp(-t\Ph^\top\Ph))(\Ph^\top\Ph)^{-1}\Ph^\top\yv,$$
where $t$ denotes the training time and $\exp$ the matrix exponential. Replacing the matrix exponential with its first-order Taylor expansion, i.e.,
$$\exp(-t\A)=\exp(t\A)^{-1}\approx (\I+t\A)^{-1},$$ 
we obtain
$$\bm{\hat{\beta}}_{\text{GF}}\approx(\I_d-(\I_p+t\Ph^\top\Ph)^{-1})(\Ph^\top\Ph)^{-1}\Ph^\top\yv,$$
which, for $t=1/\lambda$, coincides with the ridge regression solution, 
$$\bm{\hat{\beta}}_{\text{RR}}=(\Ph^\top\Ph+\lambda\I_d)^{-1}\Ph^\top\yv= (\I_p-(\I_p+1/\lambda\Ph^\top\Ph)^{-1})(\Ph^\top\Ph)^{-1}\Ph^\top\yv.$$

The equivalence of Theorem \ref{thm:opt_lbda} for gradient flow is given by Theorem \ref{thm:opt_t}.
\begin{thm}~\\
\label{thm:opt_t}
Let $\Ss_\lambda:=\Ph\Ph^+_t$ and $\ssv_\lambda:=\phsvt\Ph^+_t$, where 
$$\Ph^+_t:=(\I_p-\exp(-t\Ph^\top\Ph))(\Ph^\top\Ph)^{-1}\Ph^\top=\Ph^\top(\Ph\Ph^\top)^{-1}(\I_n-\exp(-t\Ph\Ph^\top)),$$
and let $\|\cdot\|$ denote the trace seminorm, the nuclear norm, or the Frobenius norm. Then
\begin{enumerate}[label=(\alph*)]
\item
for GCV,
$\argmin_{t\geq 0}\frac{\left\|(\I_n-\Ss_t)^\top(\I_n-\Ss_t)\right\|}{\left(\Tr(\I_n-\Ss_t)\right)^2}=0$.
\item
for in-sample MSV,\\
$\argmin_{t\geq 0}\left\|\frac1n\I_n-\frac1n\Ss^\top_t\Ss_t\right\|=$\\ $\argmin_{t\geq 0}\left\|\frac1n\I_n-\frac1n\Ss_t\right\|=\infty$.
\item
for out-of-sample MSV, if $\E(\phsv\phsvt)=\I_p$, then\\
$\argmin_{t\geq 0}\left\|\frac1n\I_n-\E_{\xv}(\ssv_t\ssvt_t)\right\|>0$.\\
If, in addition, $\|\Ph\Ph^\top\|_2<n$, then\\
$\argmin_{t\geq 0}\left\|\frac1n\I_n-\E_{\xv}(\ssv_t\ssvt_t)\right\|<\infty$.
\end{enumerate}
\end{thm}

\subsection{Theorems \ref{thm:opt_lbda} and \ref{thm:opt_t} for the Spectral Norm}
Theorem \ref{thm:opt_lbda2} is the equivalent of Theorems \ref{thm:opt_lbda} and \ref{thm:opt_t} for the spectral norm. The results are the same as for the other three (semi)norms, with two exceptions:
For in-sample MSV, and $n>p$, any value of $\lambda$ or $t$ is optimal, which is equally useless as always selecting $\lambda=0$ or $t=\infty$. For out-of-sample MSV, a very reasonable additional assumption is needed when $n>p$.
\begin{thm}~\\
\label{thm:opt_lbda2}
Let
\begin{equation*}
\begin{aligned}
\Ph^+_\lambda:=&(\Ph^\top\Ph+\lambda\I_p)^{-1}\Ph^\top=\Ph^\top(\Ph\Ph^\top+\lambda\I_n)^{-1}\\
\Ph^+_t:=&(\I_p-\exp(-t\Ph^\top\Ph))(\Ph^\top\Ph)^{-1}\Ph^\top =\Ph^\top(\Ph\Ph^\top)^{-1}(\I_n-\exp(-t\Ph\Ph^\top)),
\end{aligned}
\end{equation*}
so that the in- and out-of-sample smoothers are $\Ss_\lambda:=\Ph\Ph^+_\lambda$ and $\ssv_\lambda:=\phsvt\Ph^+_\lambda$ for ridge regression in feature space, and $\Ss_t:=\Ph\Ph^+_t$ and $\ssv_t:=\phsvt\Ph^+_t$ for gradient flow in feature space. Then
\begin{enumerate}[label=(\alph*)]
\item
For GCV,
$\argmin_{\lambda\geq 0}\frac{\left\|(\I_n-\Ss_\lambda)^\top(\I_n-\Ss_\lambda)\right\|_2}{\left(\Tr(\I_n-\Ss_\lambda)\right)^2}=\infty$ and $\argmin_{t\geq 0}\frac{\left\|(\I_n-\Ss_t)^\top(\I_n-\Ss_t)\right\|_2}{\left(\Tr(\I_n-\Ss_t)\right)^2}=0$.
\item
For in-sample MSV,\\
if $\Ph\Ph^\top$ is non-singular, then\\ $\argmin_{\lambda\geq 0}\left\|\frac1n\I_n-\frac1n\Ss^\top_\lambda\Ss_\lambda\right\|_2=\argmin_{\lambda\geq 0}\left\|\frac1n\I_n-\frac1n\Ss_\lambda\right\|_2=0$ and \\
$\argmin_{t\geq 0}\left\|\frac1n\I_n-\frac1n\Ss^\top_t\Ss_t\right\|_2=\argmin_{t\geq 0}\left\|\frac1n\I_n-\frac1n\Ss_t\right\|_2=\infty$.\\
if $\Ph\Ph^\top$ is singular, then $\left\|\frac1n\I_n-\frac1n\Ss^\top_\lambda\Ss_\lambda\right\|_2=\left\|\frac1n\I_n-\frac1n\Ss_\lambda\right\|_2=\left\|\frac1n\I_n-\frac1n\Ss^\top_t\Ss_t\right\|_2=\left\|\frac1n\I_n-\frac1n\Ss_t\right\|_2=1$ for all $\lambda$ and $t$.
\item
For out-of-sample MSV, if $\E(\phsv\phsvt)=\I_p$,\\
if $\Ph\Ph^\top$ is non-singular, then
$\argmin_{\lambda\geq 0}\left\|\frac1n\I_n-\E(\ssv_\lambda\ssvt_\lambda)\right\|_2<\infty$ 
and $\argmin_{t\geq 0}\left\|\frac1n\I_n-\E(\ssv_t\ssvt_t)\right\|_2>0$.\\
if $\Ph\Ph^\top$ is singular, then
$\left\|\frac1n\I_n-\E(\ssv_\lambda\ssvt_\lambda)\right\|_2=\frac1n$ for $\lambda\geq\max_{i=1,\dots n}\left(\sqrt{\frac n2}s_i-s_i^2\right)$ and\\
$\left\|\frac1n\I_n-\E(\ssv_t\ssvt_t)\right\|_2=\frac1n$ for $t\leq\min_{i=1,\dots n}\left(\log\left(\frac{\sqrt n}{\sqrt n-\sqrt 2s_i}\right)\cdot\frac1{s_i^2}\right)$, where $\{s_i\}_{i=1}^n$ denote the singular values of $\Ph$.\\
If, in addition, $\|\Ph\Ph^\top\|_2<n$, and
$\Ph\Ph^\top$ is non-singular, or $\Ph\Ph^\top$ is singular, and at least one of its singular values is in $(0,n/2)$,
then $\argmin_{\lambda\geq 0}\left\|\frac1n\I_n-\E(\ssv_\lambda\ssvt_\lambda)\right\|_2>0$  and $\argmin_{t\geq 0}\left\|\frac1n\I_n-\E(\ssv_t\ssvt_t)\right\|_2<\infty$.
\end{enumerate}
\end{thm}

\section{PROOFS}
\label{sec:proofs}

\begin{proof}[Proof of Equation \ref{eq:t_semi_norm}]
\begin{equation*}
\begin{aligned}
\left|\E_{y}\left(\yvr^\top\A\yvr\right)\right|
\stackrel{(a)}=\left|\Tr\left(\E_{y}\left(\yvr^\top\A\yvr\right)\right)\right|
\stackrel{(b)}=\left|\Tr\left(\E_{y}\left(\yvr\yvr^\top\right)\cdot\A\right)\right|
\stackrel{(c)}=\sigma_y^2\cdot\left|\Tr\left(\A\right)\right|,
\end{aligned}
\end{equation*}
where we have used \emph{(a)} that the trace of a scalar is the scalar itself, \emph{(b)} the cyclic property of the trace, and \emph{(c)} that $\E(\yvr\yvr^\top)=\sigma_y^2\I_n$.
\end{proof}

\begin{proof}[Proof of Proposition \ref{thm:mar_pas1}]~\\
For a linear smoother, $\hat{f}^*=\ssvt\yv$, where $\yv=\fv+\bm{\varepsilon}$, standard bias-variance calculations yield

\begin{equation*}
\begin{aligned}
V_{\X}&=\E_{x^*}\left(\E_\varepsilon((\fhs-\E_\varepsilon(\fhs))^2)=\E_\varepsilon(\fhs{^2})-\E_\varepsilon(\fhs)^2\right)\\
&=\E_{x^*}\left(\E_\varepsilon\left(\ssvt\yv\yv^\top\ssv\right)-\E_\varepsilon\left(\ssvt\yv)\right)^2\right)
=\E_{x^*}\left(\E_\varepsilon\left(\ssvt(\fv+\epsv)(\fv+\epsv)^\top\ssv\right)-\E_\varepsilon\left(\ssvt(\fv+\epsv)\right)^2\right)\\
&=\E_{x^*}\left(\ssvt\fv\fv^\top\ssv+\ssvt\fv\E_\varepsilon(\epsv^\top)\ssv+\ssvt\E_\varepsilon(\epsv)\fv^\top\ssv+\ssvt\E_\varepsilon(\epsv\epsv^\top)\ssv-\left(\ssvt\fv+\ssvt\E_\varepsilon(\epsv)\right)^2\right)\\
&=\E_{x^*}\left(\ssvt\E_\varepsilon(\epsv\epsv^\top)\ssv\right)
=\E_{x^*}\left(\ssvt\ssv\cdot\sigma_\varepsilon^2\right)
=\E_{x^*}\left(\|\ssv\|_2^2\right)\cdot\sigma^2_\varepsilon.
\end{aligned}
\end{equation*}

Thus,
\begin{equation*}
\begin{aligned}
\thetahv_T&=\argmin_{\thetahv}\left|1-\E(\|\ssv(\thetahv)\|_2^2)\right|
=\argmin_{\thetahv}\left|1-V_{\X}(\thetahv)/\sigma^2_\varepsilon\right|
=\argmin_{\thetahv}\left|\sigma^2_\varepsilon-V_{\X}(\thetahv)\right|.
\end{aligned}
\end{equation*}
\end{proof}

\begin{proof}[Proof of Theorem \ref{thm:mar_pas2}]~\\
To alleviate notation, we denote the SNR as $s$ during this proof. According to Corollary 5 by \cite{hastie2022surprises}, almost surely
\begin{equation}
\label{eq:vr_as}
\begin{aligned}
V_{\X}(\lambda,\gamma)&\to \sigma^2_\varepsilon\gamma\left(m_F(-\lambda,\gamma)-\lambda m'_F(-\lambda,\gamma)\right)=:\overline{V_{\X}}(\lambda)\\
R_{\X}(\lambda,\gamma)&\to \sigma^2_\varepsilon\gamma\left(m_F(-\lambda,\gamma)-\lambda (1-s\cdot\lambda/\gamma)m'_F(-\lambda,\gamma)\right)=:\overline{R_{\X}}(\lambda),
\end{aligned}
\end{equation}
where $m_F(z,\gamma)$ and $m'_F(z,\gamma):=\partial_zm(z,\gamma)$ are the Stieltjes transform of the Marchenko-Pastur law of $\Sigh$, and its derivative (with respect to $z$). For isotropic features, when $\Sig=\I$, we obtain the closed forms

\begin{equation*}
\begin{aligned}
m_F(z,\gamma)&= \frac{1-\gamma-z-\sqrt{(1-\gamma-z)^2-4\gamma z}}{2 \gamma z}\\
m'_F(z,\gamma)&=\frac{\frac{z(1+\gamma - z)}{\sqrt{(1-\gamma - z)^2-4 \gamma z}} - 1+\gamma  + \sqrt{(1-\gamma - z)^2-4\gamma z }}{2\gamma z^2}.
\end{aligned}
\end{equation*}
Plugging this into Equation \ref{eq:vr_as}, we obtain
\begin{equation*}
\begin{aligned}
&\overline{V_{\X}}(\lambda,\gamma)=\sigma^2_\varepsilon\gamma\left(m_F(-\lambda,\gamma)-\lambda m'_F(-\lambda,\gamma)\right)
=\frac{\sigma^2_\varepsilon}2\left(\frac{1+\gamma +\lambda}{\sqrt{(1-\gamma + \lambda)^2 + 4 \gamma \lambda}} -1\right).
\end{aligned}
\end{equation*}
Solving for $\lambda$, we obtain

\begin{equation*}
\begin{aligned}
\frac{\sigma^2_\varepsilon}2\left(\frac{1+\gamma +\lambda}{\sqrt{(1-\gamma + \lambda)^2 + 4 \gamma \lambda}} -1\right)=\sigma^2_\varepsilon\iff \lambda=3\sqrt{\frac\gamma 2}-\gamma-1\\
\text{where }\lambda\geq 0 \text{ for } \frac12\leq \gamma\leq 2 \implies \overline{\lambda_T}=
\begin{cases}
3\sqrt{\frac\gamma 2}-\gamma-1, & \gamma \in \left(\frac12,2\right)\\
0, & \gamma \in \left(0,\frac12\right]\cup\left[2,\infty\right)
\end{cases}
\end{aligned}
\end{equation*}
and 
\begin{equation*}
\begin{aligned}
&\overline{R_{\X}}(\overline{\lambda_T},\gamma)=
\begin{cases}
\sigma^2_\varepsilon\cdot\left(1+s\frac{(\sqrt{2\gamma}-1)^2}{\gamma}\right), & \gamma \in\left(\frac12,2\right)\\
\overline{R_{\X}}(0,\gamma), & \gamma\in\left(0,\frac12\right]\cup\left[2,\infty\right).
\end{cases}
\end{aligned}
\end{equation*}

According to Lemma \ref{thm:drdgamma1}, the maximum of $\asymrisk(\overline{\lambda_T},\gamma)/\asymrisk(\overline{\lambda^*},\gamma)$ occurs for some $\gamma \in\left[\frac12,2\right]$. 
We make a grid in $\gamma\in\left[\frac12,2\right]$ and $s\in[1,80]$ with $\Delta\gamma=\Delta s=10^{-5}$, and numerically calculate the maximum of $\asymrisk(\overline{\lambda_T},\gamma)/\asymrisk(\overline{\lambda^*},\gamma)$ on this grid. The maximum value is less than 2.445. 

We then use the bounds on $\left|\frac{\partial \asymrisk(\overline{\lambda_T},\gamma)/\asymrisk(\overline{\lambda^*}(s),\gamma)}{\partial \gamma}\right|$ and $\left|\frac{\partial \asymrisk(\overline{\lambda_T},\gamma)/\asymrisk(\overline{\lambda^*}(s),\gamma)}{\partial s}\right|$ from Lemma \ref{thm:drdgamma2} to bound the maximum deviation of the function value between the grid points. The deviation is less than
\begin{equation*}
\begin{aligned}
&\frac12\cdot\sqrt{\left(\max_{\gamma\in\left[\frac12,2\right]}\left|\frac{\partial \asymrisk(\overline{\lambda_T},\gamma)/\asymrisk(\overline{\lambda^*}(s),\gamma)}{\partial \gamma}\right|\right)^2\cdot(\Delta\gamma)^2+\left(\max_{s\in[1,80]}\left|\frac{\partial \asymrisk(\overline{\lambda_T},\gamma)/\asymrisk(\overline{\lambda^*}(s),\gamma)}{\partial s}\right|\right)^2\cdot(\Delta s)^2}\\
&\leq \frac{10^{-5}}2\cdot\sqrt{752.2^2+10.5^2}\leq 0.003762.
\end{aligned}
\end{equation*}
Thus, the maximum is less than $2.445+0.003762\leq 2.449$.

\end{proof}

\begin{lemma}
\label{thm:drdgamma1}
\begin{equation*}
\begin{aligned}
\frac{\partial \asymrisk(\overline{\lambda_T},\gamma)/\asymrisk(\overline{\lambda^*},\gamma)}{\partial \gamma}
\begin{cases}
>0 \text{ for } \gamma\in\left(0,\frac12\right]\\
<0 \text{ for } \gamma\in\left[2,\infty\right).
\end{cases}
\end{aligned}
\end{equation*}
\end{lemma}

\begin{lemma}
\label{thm:drdgamma2}
For $\gamma \in \left[\frac12, 2\right]$, and $s\in[1,80]$,
\begin{equation*}
\left|\frac{\partial \asymrisk(\overline{\lambda_T},\gamma)/\asymrisk(\overline{\lambda^*}(s),\gamma)}{\partial \gamma}\right|\leq 752.2
\end{equation*}
and
\begin{equation*}
\left|\frac{\partial \asymrisk(\overline{\lambda_T},\gamma)/\asymrisk(\overline{\lambda^*}(s),\gamma)}{\partial s}\right|\leq 10.5.
\end{equation*}
\end{lemma}

\begin{proof}[Proof of Lemma \ref{thm:drdgamma1}]~\\

According to \cite{hastie2022surprises},
\begin{equation*}
\begin{aligned}
&\asymrisk(0,\gamma)=
\begin{cases}
\sigma^2_\varepsilon\cdot\frac\gamma{1-\gamma}, & \gamma\leq1\\
\sigma^2_\varepsilon\cdot\left(s\left(1-\frac1\gamma\right)+\frac1{\gamma-1}\right), & \gamma>1 
\end{cases}\\
&\asymrisk(\overline{\lambda^*},\gamma)=\frac{\sigma^2_\varepsilon}2\left(s-s/\gamma-1+\sqrt{4s+(1-s+s/\gamma)^2}\right),\\
\end{aligned}
\end{equation*}
and straightforward calculations yield
$$\frac{\partial\left(\asymrisk(0,\gamma)/\asymrisk(\overline{\lambda^*},\gamma)\right)}{\partial \gamma}=
\begin{cases}
s\cdot\frac{1+s\cdot(\frac1\gamma-1)  -  \sqrt{4s + \left(1+ s\cdot(\frac1\gamma-1)\right)^{2}}  }{2\gamma^{2} \sqrt{4s + \left(1+ s\cdot(\frac1\gamma-1)\right)^{2}}} + \frac{1}{\left(\gamma - 1\right)^{2}}, &\text{ for } \gamma<1\\
s\cdot\frac{1-s\cdot(1-\frac1\gamma)  +  \sqrt{4s + \left(1- s\cdot(1-\frac1\gamma)\right)^{2}}  }{2\gamma^{2} \sqrt{4s + \left(1- s\cdot(1-\frac1\gamma)\right)^{2}}} - \frac{1}{\left(1-\gamma\right)^{2}},&\text{ for } \gamma>1.
\end{cases}
$$

We first show that, for $\gamma<0$,
$$s\cdot\frac{1+s\cdot(\frac1\gamma-1)  -  \sqrt{4s + \left(1+ s\cdot(\frac1\gamma-1)\right)^{2}}  }{2\gamma^{2} \sqrt{4s + \left(1+ s\cdot(\frac1\gamma-1)\right)^{2}}} + \frac{1}{\left(\gamma - 1\right)^{2}}>0,$$
using the intermediate variables $a:=\frac1\gamma-1>0$, $b:=1+sa>0$ and $c:=\sqrt{4s+b^2}>0$:

\begin{equation*}
\begin{aligned}
&s\cdot\frac{1+s\cdot(\frac1\gamma-1)  -  \sqrt{4s + \left(1+ s\cdot(\frac1\gamma-1)\right)^{2}}  }{2\gamma^{2} \sqrt{4s + \left(1+ s\cdot(\frac1\gamma-1)\right)^{2}}} + \frac{1}{\gamma^2a^2}\\
&=s\cdot\frac{b - c}{2\gamma^{2} c} + \frac{1}{\gamma^2a^2}
=s\cdot\frac{b - c}{2\gamma^{2} c}\cdot\frac{b+c}{b+c} + \frac{1}{\gamma^2a^2}
=s\cdot\frac{b^2 - c^2}{2\gamma^{2} c(b+c)} + \frac{1}{\gamma^2a^2}\\
&=\frac{-2s^2}{\gamma^{2} c(b+c)} + \frac{1}{\gamma^2a^2}.
\end{aligned}
\end{equation*}
Since $c\geq b > 0$, and $b>s a$, we have $c(b+c)\geq b(b+b)=2b^2> 2s^2a^2$, and
\begin{equation*}
\begin{aligned}
\frac{1}{\gamma^2a^2}- \frac{2s^2}{\gamma^{2} c(b+c)}
>\frac{1}{\gamma^2a^2} -\frac{2s^2}{\gamma^{2} \cdot 2s^2 a^2}
=\frac1{\gamma^2a^2} - \frac{1}{\gamma^2a^2}=0.
\end{aligned}
\end{equation*}


We then show that, for $\gamma>1$,
$$s\cdot\frac{1-s\cdot(1-\frac1\gamma)  +  \sqrt{4s + \left(1- s\cdot(1-\frac1\gamma)\right)^{2}}  }{2\gamma^{2} \sqrt{4s + \left(1- s\cdot(1-\frac1\gamma)\right)^{2}}} - \frac{1}{\left(\gamma -1\right)^{2}}<0.$$
Using the intermediate variables $a:=1-\frac1\gamma\in (0,1)$, $b:=1-s\cdot a<1$ and $c:=\sqrt{4s+b^2}>0$, we obtain

\begin{equation*}
\begin{aligned}
&s\cdot\frac{1-s\cdot(1-\frac1\gamma)  +  \sqrt{4s + \left(1- s\cdot(1-\frac1\gamma)\right)^{2}}  }{2\gamma^{2} \sqrt{4s + \left(1- s\cdot(1-\frac1\gamma)\right)^{2}}} - \frac{1}{\left(\gamma-1\right)^{2}}
&=\frac{s(b + c)}{2\gamma^{2} c} - \frac{1}{\gamma^2a^2}\\
&=\frac1{2\gamma^2c}\left(sb+c\left(s-\frac2{a^2}\right)\right).
\end{aligned}
\end{equation*}
We thus need to show that $$\underbrace{\left(\frac2{a^2}-s\right)}_{=:d}c>sb.$$
We split the analysis into the following three cases: $s\in \left(0,\frac1a\right]$, $s\in \left(\frac1a,\frac2{a^2}\right]$, and $s\in \left(\frac2{a^2},\infty\right)$.

For $s\in \left(0,\frac1a\right]$, we have $b\geq0$, $d>0$, and $c>0$, and, using $c^2=4s+b^2$, we obtain
$$dc>sb\iff d^2c^2> s^2b^2\iff d^2(4s+b^2)>s^2b^2\iff b^2(d^2-s^2)+4sd^2>0.$$
Using the conjugate rule and the definition of $d$, we further obtain
$$b^2(d^2-s^2)+4sd^2=b^2(d+s)(d-s)+4sd^2=b^2\cdot\frac2{a^2}\cdot\left(\frac2{a^2}-2s\right)+4sd^2.$$
Since $4sd^2>0$, we just need to show that $\left(\frac2{a^2}-2s\right)\geq0$. Since $s\leq \frac1a$, and $a\in(0,1)$,
$$\left(\frac2{a^2}-2s\right)=2\cdot\left(\frac1{a^2}-s\right)\geq2\cdot\left(\frac1{a^2}-\frac1a\right)=2\cdot\frac{1-a}{a^2}>0.$$

For $s\in \left(\frac1a,\frac2{a^2}\right]$, we have $b<0$, $d\geq0$, and $c>0$. Thus, trivially $cd>sb$.

For $s\in \left(\frac2{a^2},\infty\right)$, we have $b<0$, $d<0$, and $c>0$, and 
$$dc>sb\iff d^2c^2< s^2b^2\iff s^2b^2-d^2c^2>0.$$
Using the definitions of $b$, $c$, and $d$, we obtain

\begin{equation*}
\begin{aligned}
f(s)&:=s^2b^2-d^2c^2=s^2(1-sa)^2-\left(\frac2{a^2}-s\right)^2\left(4s+(1-sa)^2\right)\\
&=\frac4{a^4}\left((3 - 2 a) a^2 s^2 + (a^2 + 2 a - 4) s - 1\right).
\end{aligned}
\end{equation*}
Since $f(x)$ has a positive second derivative, if $f(2/a^2)>0$ and $f'(2/a^2)>0$, then $f(s)>0$ for $s\geq 2/a^2$. Straightforward calculations yield
\begin{equation*}
\begin{aligned}
f(2/a^2)&=\frac4{a^6}(a-2)^2>0,\\
f'(2/a^2)&=\frac4{a^4}\left(2(3 - 2 a) a^2 s + a^2 + 2 a - 4\right),\ f'(2/a^2)=\frac4{a^4}\left(a^2-6a+8\right)>\frac4{a^4}\left(a^2-6+8\right)>0.
\end{aligned}
\end{equation*}

\end{proof}

\begin{proof}[Proof of Lemma \ref{thm:drdgamma2}]~\\
Let 
\begin{equation*}
\begin{aligned}
f_1(\gamma)&:=1+s\cdot\frac{(\sqrt{2\gamma}-1)^2}\gamma\\
f_2(\gamma)&:=\frac12\cdot\left(s-s/\gamma-1+\sqrt{4s+(1-s+s/\gamma)^2}\right),
\end{aligned}
\end{equation*}
so that 
$\asymrisk(\overline{\lambda_T},\gamma)/\asymrisk(\overline{\lambda^*},\gamma)=f_1(\gamma)/f_2(\gamma)$.

According to the quotient rule and the triangle inequality,

\begin{equation*}
\begin{aligned}
\left|\left(\frac{f_1(\gamma)}{f_2(\gamma)}\right)'\right|
&=\left|\frac{f'_1(\gamma)\cdot f_2(\gamma)-f_1(\gamma)\cdot f'_2(\gamma)}{f_2^2(\gamma)}\right|
\leq\left|\frac{f'_1(\gamma)\cdot f_2(\gamma)}{f_2^2(\gamma)}\right|+\left|\frac{f_1(\gamma)\cdot f'_2(\gamma)}{f_2^2(\gamma)}\right|\\
&=|f'_1(\gamma)|\cdot\frac1{f_2(\gamma)}+\frac{|f'_2(\gamma)|}{f_2(\gamma)}\cdot\frac{f_1(\gamma)}{f_2(\gamma)},
\end{aligned}
\end{equation*}
where we have used that $f_1,f_2\geq 0$.
According to straightforward calculations,
\begin{equation*}
\begin{aligned}
f'_1(\gamma)&=s\cdot\frac{\sqrt{2\gamma}-1}{\gamma^2}\\
f'_2(\gamma)&=\frac{s}{2\gamma^2}\cdot\left(1-\frac{1-s+s/\gamma}{\sqrt{4s+(1-s+s/\gamma)^2}}\right)
=\frac{s}{\gamma^2\sqrt{4s+(1-s+s/\gamma)^2}}\cdot f_2(\gamma).
\end{aligned}
\end{equation*}
Let us first bound $|f'_1(\gamma)|$. For $\gamma \in \left[\frac12,2\right]$, $|f'_1(\gamma)|=f'_1(\gamma)$ and 
\begin{equation*}
\begin{aligned}
0=f''_1(\gamma)=s\cdot\frac{4-3\sqrt{2\gamma}}{2\gamma^3}\iff \gamma\in\left\{\frac89,\infty\right\}
\end{aligned}
\end{equation*}
Since $f'_1(8/9)=27/64=0.421875>f'_1(2)=1/4>f'_1(1/2)=0$, we have $|f'_1(\gamma)|=f'_1(\gamma)\leq 27/64$.

Let us now bound $|f'_2(\gamma)|/f_2(\gamma)=s/\left(\gamma^2\sqrt{4s+(1-s+s/\gamma)^2}\right)$.

$$\gamma^2\geq 1/4 \text{ and } \sqrt{4s+(1-s+s/\gamma)^2}\geq\sqrt{4s}=2\sqrt{s}.$$
Thus
$$\frac{|f'_2(\gamma)|}{f_2(\gamma)}=\frac{s}{\left(\gamma^2\sqrt{4s+(1-s+s/\gamma)^2}\right)}\leq \frac{s}{1/4\cdot2\sqrt{s}}=2\sqrt{s}.$$

Thus, so far we have
\begin{equation*}
\begin{aligned}
\left|\left(\frac{f_1(\gamma)}{f_2(\gamma)}\right)'\right|
&\leq\frac{\frac{27}{64}+2\sqrt{s}\cdot f_1(\gamma)}{f_2(\gamma)}.
\end{aligned}
\end{equation*}
Since $f'_1(\gamma)>0$ for $s>0$ and $\gamma>1/2$, $f_1(\gamma)$ obtains its maximum for $\gamma=2$, i.e., $f_1(\gamma)\leq 1+s/2$.
Since $f'_2(\gamma)>0$ for $s>0$ and $\gamma>0$, $f_2(\gamma)$ obtains its minimum for $\gamma=1/2$, i.e., 

\begin{equation*}
\begin{aligned}
f_2(\gamma)&\geq \frac{\sqrt{s^2+6s+1}-s-1}{2}
=\frac{\left(\sqrt{s^2+6s+1}-s-1\right)\cdot\left(\sqrt{s^2+6s+1}+s+1\right)}{2\cdot\left(\sqrt{s^2+6s+1}+s+1\right)}
=\frac{s^2+6s+1-(s+1)^2}{2\cdot\left(\sqrt{s^2+6s+1}+s+1\right)}\\
&\geq\frac{s^2+6s+1-(s+1)^2}{2\cdot\left(\sqrt{s^2+6s+9}+s+1\right)}
=\frac{4s}{2\cdot\left(\sqrt{(s+3)^2}+s+1\right)}
=\frac{s}{s+2}.
\end{aligned}
\end{equation*}

Putting it together, we obtain

$$ \left|\left(\frac{f_1(\gamma)}{f_2(\gamma)}\right)'\right|
\leq\left(\frac{27}{64}+2\sqrt{s}\cdot (1+s/2)\right)\cdot\frac{s+2}{s}
=\left(\frac{27}{64}+\sqrt{s}\cdot (s+2)\right)\cdot\frac{s+2}{s}.$$
Since $\left(\frac{27}{64}+\sqrt{s}\cdot (s+2)\right)\cdot\frac{s+2}{s}>0$ for $s>0.8757$, it obtains it maximum at $s=80$, where it is $752.198$.

Now, we instead let 

\begin{equation*}
\begin{aligned}
f_1(s)&:=1+s\cdot\frac{(\sqrt{2\gamma}-1)^2}\gamma\\
f_2(s)&:=\frac12\cdot\left(s-s/\gamma-1+\sqrt{4s+(1-s+s/\gamma)^2}\right),
\end{aligned}
\end{equation*}
so that 
$\asymrisk(\overline{\lambda_T},\gamma)/\asymrisk(\overline{\lambda^*},\gamma)=f_1(s)/f_2(s)$.

According to the quotient rule and the triangle inequality,

\begin{equation*}
\begin{aligned}
\left|\left(\frac{f_1(s)}{f_2(s)}\right)'\right|
&=\left|\frac{f'_1(s)\cdot f_2(s)-f_1(s)\cdot f'_2(s)}{f_2^2(s)}\right|
\leq\left|\frac{f'_1(s)\cdot f_2(s)}{f_2^2(s)}\right|+\left|\frac{f_1(s)\cdot f'_2(s)}{f_2^2(s)}\right|\\
&=|f'_1(s)|\cdot\frac1{f_2(s)}+|f'_2(s)|\cdot\frac{f_1(s)}{f^2_2(s)},
\end{aligned}
\end{equation*}
where we have again used that $f_1,f_2\geq 0$.
According to straightforward calculations,
\begin{equation*}
\begin{aligned}
f'_1(s)&=\frac{(\sqrt{2\gamma}-1)^2}{\gamma}\\
f'_2(s)&=\frac{1-(1/\gamma-1)\cdot f_2(s)}{\sqrt{4s+(1-s+s/s)^2}}.
\end{aligned}
\end{equation*}
Let us first bound $|f'_1(s)|$. For $\gamma \in \left[\frac12,2\right]$, $|f'_1(s)|=f'_1(s)\geq 0$, so the maximum is obtained for $\gamma=2$. Thus
$$|f'_1(s)|\leq\frac{(\sqrt{2\cdot 2}-1)^2}{2}=\frac12.$$

Let us now bound $|f'_2(\gamma)|=\left(1-(1/\gamma-1)\cdot f_2(s)\right)/\sqrt{4s+(1-s+s/\gamma)^2}$. For $\gamma\in \left[\frac12, 2\right]$,
$$1/\gamma-1\in\left[-\frac12,1\right] \text{ and } \sqrt{4s+(1-s+s/\gamma)^2}\geq\sqrt{4s}=2\sqrt{s}.$$
Thus
$$|f'_2(s)|=\left|\frac{1-(1/\gamma-1)\cdot f_2(s)}{\sqrt{4s+(1-s+s/s)^2}}\right|\leq \frac{1+f_2(s)}{2\sqrt{s}}.$$

From above, we know that $f_1(s)\leq 1+s/2$, and $f_2(s)\geq s/(s+2)$.

Putting it together, we obtain

\begin{equation*}
\begin{aligned}
\left|\left(\frac{f_1(s)}{f_2(s)}\right)'\right|
&\leq|f'_1(s)|\cdot\frac1{f_2(s)}+|f'_2(s)|\cdot\frac{f_1(s)}{f^2_2(s)}
\leq\frac12\cdot\frac1{f_2(s)}+\frac{1+f_2(s)}{2\sqrt{s}}\cdot\frac{1+s/2}{f^2_2(s)}\\
&=\frac12\cdot\frac{s+2}{s}+\frac{1+\frac{s+2}{s}}{4\sqrt{s}}\cdot(2+s)\cdot\frac{(s+2)^2}{s^2}
=\frac{s+2}{2s}\cdot\left(1+\frac{(s+1)(s+2)}{s\sqrt{s}}\right),
\end{aligned}
\end{equation*}
which, for $s\in[1,80]$, obtains it maximum for $s=1$, where its value is $10.5$.

\end{proof}

\begin{proof}[Proof of Theorems \ref{thm:opt_lbda}, \ref{thm:opt_t}, and \ref{thm:opt_lbda2}]~\\
Let $\{s^2_i\}_{i=1}^n$ denote the singular values of $\Ph\Ph^\top$, where, if $n>p$, $s_i=0$ for $i>p$, and where
$\{s_i\}_{i=1}^{\min(n,p)}$ are the singular values of $\Ph$.

We first note that when a matrix $\A$ is positive semi-definite, its eigenvalues and singular values coincide, and thus $\|\A\|_T=|\Tr(\A)|=\|\A\|_*$. This is the case for parts $(a)$ and $(b)$.

We will first prove all ridge regression expressions, i.e.\ expressions depending on $\lambda$. The proofs for gradient flow, which depend on $t$, are very similar, and done afterwards.\\

\textbf{Ridge Regression ($\lambda$)}\\
\emph{Part (a)}\\
For part \emph{(a)}, let $a_i:=1-\frac{s_i^2}{s_i^2+\lambda}\geq 0$. Then, the singular values of 
$(\I_n-\Ss_\lambda)^\top(\I_n-\Ss_\lambda)$ and $\I_n-\Ss_\lambda$ are $\left\{\left(1-\frac{s_i^2}{s_i^2+\lambda}\right)^2\right\}_{i=1}^n=\left\{a_i^2\right\}_{i=1}^n$, and $\left\{1-\frac{s_i^2}{s_i^2+\lambda}\right\}_{i=1}^n=\left\{a_i\right\}_{i=1}^n$ respectively. Since the $a_i\geq 0$ for all $i$, $\I_n-\Ss_\lambda$ is positive semi-definite, which means that its singular and eigenvalues coincide and thus that its trace equals the sum of its singular values. Consequently

\begin{equation}
\label{eq:gcv_mses}
\begin{aligned}
\frac{\left\|(\I_n-\Ss_\lambda)^\top(\I_n-\Ss_\lambda)\right\|}{\left(\Tr(\I_n-\Ss_\lambda)\right)^2}
&=\begin{cases}
\frac{\sum_{i=1}^na_i^2}{\left(\sum_{i=1}^na_i\right)^2}\text{ for the nuclear (and trace semi) norm.}\\
\sqrt{\frac{\sum_{i=1}^na_i^4}{\left(\sum_{i=1}^na_i\right)^4}}\text{ for the Frobenius norm.}\\
\frac{\max_{i=1,\dots n}a_i^2}{\left(\sum_{i=1}^na_i\right)^2}\text{ for the spectral norm.}\\
\end{cases}
\end{aligned}
\end{equation}
For the nuclear and Frobenius norms, we can use that for any vector, $\vv\in \R^n$, for $d>1$,
\begin{equation*}
\|\vv\|_d=\left(\sum_{i=1}^n|v_i|^d\right)^{1/d}\geq n^{1-1/d}\cdot\|\vv\|_1=n^{1-1/d}\cdot\sum_{i=1}^n|v_i|,
\end{equation*}
with equality iff $v_i=1$ for all $i$. Thus the two corresponding expressions in Equation \ref{eq:gcv_mses} are minimized when $a_i=1$ for all $i$, i.e.\ for $\lambda=\infty$, in which case their values are

\begin{equation*}
\begin{aligned}
\lim_{\lambda\to \infty}\left(\frac{\left\|(\I_n-\Ss_\lambda)^\top(\I_n-\Ss_\lambda)\right\|_*}{\left(\Tr(\I_n-\Ss_\lambda)\right)^2}\right)=\frac1n\\
\lim_{\lambda\to \infty}\left(\frac{\left\|(\I_n-\Ss_\lambda)^\top(\I_n-\Ss_\lambda)\right\|_F}{\left(\Tr(\I_n-\Ss_\lambda)\right)^2}\right)=\frac1{n\sqrt{n}}.
\end{aligned}
\end{equation*}
For the spectral norm, let $a_m:=\max_{i=1,\dots n}a_i$. Then the corresponding expression in Equation \ref{eq:gcv_mses} becomes $a_m^2/\left(\sum_{i=1}^na_i\right)^2=1/\left(\sum_{i=1}^n\frac{a_i}{a_m}\right)^2$. Here, $\frac{a_i}{a_m}\leq 1$ is maximized when $a_i=a_m$ for all $i$, i.e. when $\lambda=\infty$ and $a_i=1$ for all $i$.
In this case,
\begin{equation*}
\lim_{\lambda\to \infty}\left(\frac{\left\|(\I_n-\Ss_\lambda)^\top(\I_n-\Ss_\lambda)\right\|_2}{\left(\Tr(\I_n-\Ss_\lambda)\right)^2}\right)=\frac1{n^2}.
\end{equation*}

~\\

\emph{Part (b)}\\
For part \emph{(b)},
\begin{equation*}
\begin{aligned}
&n\cdot\left\|\frac1n\I_n-\frac1n\Ss^\top_\lambda\Ss_\lambda\right\|\\
=&n\cdot\left\|\frac1n\I_n-\frac1n\Ph(\Ph^\top\Ph+\lambda\I_p)^{-1}\Ph^\top\Ph(\Ph^\top\Ph+\lambda\I_p)^{-1}\Ph^\top\right\|\\
=&\begin{cases}
\sum_{i=1}^n\left|1-\frac{s_i^4}{(s_i^2+\lambda)^2}\right|\text{ for the nuclear (and trace semi) norm.}\\
\sum_{i=1}^n\left(1-\frac{s_i^4}{(s_i^2+\lambda)^2}\right)^2\text{ for the squared Frobenius norm.}\\
\max_{i=1,\dots n}\left|1-\frac{s_i^4}{(s_i^2+\lambda)^2}\right|\text{ for the spectral norm.}\\
\end{cases}
\end{aligned}
\end{equation*}
If $s_i>0$, $\frac{s_i^4}{(s_i^2+\lambda)^2}\in[0,1]$ for $\lambda\geq 0$, and equals 1 only for $\lambda=0$. Hence, $\left|1-\frac{s_i^4}{(s_i^2+\lambda)^2}\right|$ (and thus $\left(1-\frac{s_i^4}{(s_i^2+\lambda)^2}\right)^2$) is minimized for $\lambda=0$, in which case it is 0.
If $s_i=0$, $\left|1-\frac{s_i^4}{(s_i^2+\lambda)^2}\right|= 1$ regardless of $\lambda$. 
For the nuclear and Frobenius norms, since each term of the sum obtains its minimum for $\lambda=0$, so does the sum.
For the spectral norm, the minimum is obtained either for $\lambda=0$, if $\argmax_{s_i}\left|1-\frac{s_i^4}{(s_i^2+\lambda)^2}\right|>0$, or for any value of $\lambda$ if $\argmax_{s_i}\left|1-\frac{s_i^4}{(s_i^2+\lambda)^2}\right|=0$.

Analogously,
\begin{equation*}
\begin{aligned}
n\cdot\left\|\frac1n\I_n-\frac1n\Ss_\lambda\right\|
&=n\cdot\left\|\frac1n\I_n-\frac1n\Ph(\Ph^\top\Ph+\lambda\I_p)^{-1}\Ph^\top\right\|\\
&=\begin{cases}
\sum_{i=1}^n\left|1-\frac{s_i^2}{s_i^2+\lambda}\right|\text{ for the nuclear (and trace semi) norm.}\\
\sum_{i=1}^n\left(1-\frac{s_i^2}{s_i^2+\lambda}\right)^2\text{ for the squared Frobenius norm.}\\
\max_{i=1,\dots n}\left|1-\frac{s_i^2}{s_i^2+\lambda}\right|\text{ for the spectral norm.}\\
\end{cases}
\end{aligned}
\end{equation*}

where, for $\lambda\geq 0$, $\frac{s_i^2}{s_i^2+\lambda}\in[0,1]$ equals 1 only for $\lambda=0$, if $s_i>0$,
and $\frac{s_i^2}{s_i^2+\lambda}=0$ regardless of $\lambda$ if $s_i=0$. 
Again, for the nuclear and Frobenius norms, each term of the sums obtains its minimum for $\lambda=0$, and so do the sums.
For the spectral norm, the minimum is obtained either for $\lambda=0$, if $\argmax_{s_i}\left|1-\frac{s_i^2}{s_i^2+\lambda}\right|>0$, or for any value of $\lambda$ if $\argmax_{s_i}\left|1-\frac{s_i^2}{s_i^2+\lambda}\right|=0$.

~\\

\emph{Part (c)}\\
For part \emph{(c)}, we want to show that the derivative of $\left\|\frac1n\I-\E(\ssv_\lambda\ssvt_\lambda)\right\|$ with respect to $\lambda$ is strictly negative for $\lambda=0$, strictly positive for $M_1<\lambda<M_2$ and non-negative for $\lambda\geq M_2$, in which case the minimum of the norm is obtained for $\lambda\in(0,M_1)$.

Since
\begin{equation*}
\frac1n\I_n-\E(\ssv_\lambda\ssvt_\lambda)=\frac1n\I_n-(\Ph\Ph^\top+\lambda\I_n)^{-1}\Ph\underbrace{\E(\phsv\phsvt)}_{=\I_p}\Ph^\top(\Ph\Ph^\top+\lambda\I_n)^{-1},
\end{equation*}

\begin{equation}
\label{eq:hs_norm} 
\begin{aligned}
&\left\|\frac1n\I_n-\E(\ssv_\lambda\ssvt_\lambda)\right\|_T=\left|\sum_{i=1}^n\left(\frac1n-\frac{s_i^2}{(s_i^2+\lambda)^2}\right)\right|\\
&\left\|\frac1n\I_n-\E(\ssv_\lambda\ssvt_\lambda)\right\|_*=\sum_{i=1}^n\left|\frac1n-\frac{s_i^2}{(s_i^2+\lambda)^2}\right|\\
&\left\|\frac1n\I_n-\E(\ssv_\lambda\ssvt_\lambda)\right\|_F=\sum_{i=1}^n\left(\frac1n-\frac{s_i^2}{(s_i^2+\lambda)^2}\right)^2\\
&\left\|\frac1n\I_n-\E(\ssv_\lambda\ssvt_\lambda)\right\|_2=\max_{i=1,\dots n}\left|\frac1n-\frac{s_i^2}{(s_i^2+\lambda)^2}\right|=\left|\frac1n-\frac{s_m^2}{(s_m^2+\lambda)^2}\right|,
\end{aligned}
\end{equation}
where $s_m:=\argmax_{s_i}\left|\frac1n-\frac{s_i^2}{(s_i^2+\lambda)^2}\right|$.

The needed derivatives with respect to $\lambda$ are
\begin{equation*}
\begin{aligned}
&\frac{\partial}{\partial \lambda}\left(\frac1n-\frac{s_i^2}{(s_i^2+\lambda)^2}\right)
=\frac{2s_i^2}{(s_i^2+\lambda)^3}\\
&\frac{\partial}{\partial \lambda}\left|\frac1n-\frac{s_i^2}{(s_i^2+\lambda)^2}\right|
=\frac{2s_i^2}{(s_i^2+\lambda)^3}\cdot\sgn\left(\frac1n-\frac{s_i^2}{(s_i^2+\lambda)^2}\right)\\
&\frac{\partial}{\partial \lambda}\left(\frac1n-\frac{s_i^2}{(s_i^2+\lambda)^2}\right)^2
=\frac{4s_i^2}{(s_i^2+\lambda)^3}\cdot\left(\frac1n-\frac{s_i^2}{(s_i^2+\lambda)^2}\right),
\end{aligned}
\end{equation*}
and hence
\begin{equation}
\label{eq:der_hs_norm}
\begin{aligned}
&\frac{\partial}{\partial \lambda}\left\|\frac1n\I_n-\E(\ssv_\lambda\ssvt_\lambda)\right\|_T
=\sum_{i=1}^n\frac{2s_i^2}{(s_i^2+\lambda)^3}\cdot\sgn\left(\sum_{i=1}^n\left(\frac1n-\frac{s_i^2}{(s_i^2+\lambda)^2}\right)\right)\\
&\frac{\partial}{\partial \lambda}\left\|\frac1n\I_n-\E(\ssv_\lambda\ssvt_\lambda)\right\|_{*}
=\sum_{i=1}^n\frac{2s_i^2}{(s_i^2+\lambda)^3}\cdot\sgn\left(\frac1n-\frac{s_i^2}{(s_i^2+\lambda)^2}\right)\\
&\frac{\partial}{\partial \lambda}\left\|\frac1n\I_n-\E(\ssv_\lambda\ssvt_\lambda)\right\|_{F}
=\sum_{i=1}^n\frac{4s_i^2}{(s_i^2+\lambda)^3}\cdot\left(\frac1n-\frac{s_i^2}{(s_i^2+\lambda)^2}\right)\\
&\frac{\partial}{\partial \lambda}\left\|\frac1n\I_n-\E(\ssv_\lambda\ssvt_\lambda)\right\|_{2}
=\frac{2s_m^2}{(s_m^2+\lambda)^3}\cdot\sgn\left(\frac1n-\frac{s_m^2}{(s_m^2+\lambda)^2}\right).
\end{aligned}
\end{equation}

We first consider large $\lambda$'s.
For the trace seminorm and the nuclear and Frobenius norm, we note that $\lambda>\sqrt{n}s_i-s_i^2\implies\left(\frac1n-\frac{s_i^2}{(s_i^2+\lambda)^2}\right)>0\implies\sum_{i=1}^n\left(\frac1n-\frac{s_i^2}{(s_i^2+\lambda)^2}\right)>0$, and hence, if $\max_{i=1,\dots n}\left( \sqrt{n}s_i-s_i^2\right)=:M_1<\lambda<\infty$, each term in the three sums in Equation \ref{eq:der_hs_norm} sum is strictly positive for $s_i>0$, and zero for $s_i=0$. Consequently, for $M_1<\lambda<\infty$,
\begin{equation*}
\begin{aligned}
&\frac{\partial}{\partial \lambda}\left\|\frac1n\I_n-\E(\ssv_\lambda\ssvt_\lambda)\right\|_{T}>0\\
&\frac{\partial}{\partial \lambda}\left\|\frac1n\I_n-\E(\ssv_\lambda\ssvt_\lambda)\right\|_{*}>0\\
&\frac{\partial}{\partial \lambda}\left\|\frac1n\I_n-\E(\ssv_\lambda\ssvt_\lambda)\right\|_{F}>0.\\
\end{aligned}
\end{equation*}
For the spectral norm, to obtain $\frac{\partial}{\partial \lambda}\left\|\frac1n\I_n-\E(\ssv_\lambda\ssvt_\lambda)\right\|_{2}>0$ when $\sqrt{n}s_m-s_m^2=:M_m<\lambda<\infty$, we also need $s_m>0$. If $\Ph\Ph^\top$ is non-singular, all its singular values are non-zero, and this is trivially fulfilled. However, if $\Ph\Ph^\top$ is singular, at least one $s_i$ is zero. We call this singular value $s_0$ (i.e. $s_0=0$) and note that $\left|\frac1n-\frac{s_0^2}{(s_0^2+\lambda)^2}\right|=\frac1n$ regardless of $\lambda$.
Furthermore, if $\lambda\geq \sqrt{\frac n2}s_i-s_i^2$ for all $s_i$, then 
$\left|\frac1n-\frac{s_i^2}{(s_i^2+\lambda)^2}\right|\leq \frac1n$. This means that for for $\lambda\geq \max_{i=1,\dots n}\sqrt{\frac n2}s_i-s_i^2$, $s_m=0$ and $\left\|\frac1n\I_n-\E(\ssv_\lambda\ssvt_\lambda)\right\|_{2}=\frac1n$.
~\\

We now consider $\lambda=0$.
For the trace seminorm and for the nuclear and Frobenius norms, if $s_i=0$, the corresponding terms in the sums in Equation \ref{eq:der_hs_norm} are zero, and thus do not contribute when calculating the derivative. For terms where $s_i>0$, $\lim_{\lambda \to 0}\left(\frac1n-\frac{s_i^2}{(s_i^2+\lambda)^2}\right)=\left(\frac1n-\frac{1}{s_i^2}\right)$.
Since $\|\Ph\Ph^\top\|_2< n$ by assumption, $s_i^2\in(0,n)$ for all $i$, and we can write $s_i^2=n-a_i$ for $0< a_i< n$. Hence,

\begin{equation*}
\left(\frac1n-\frac{1}{s_i^2}\right)
=\left(\frac1n-\frac1{n-a_i}\right)=\left(-\frac{a_i}{n(n-a_i)}\right)<0.
\end{equation*}
Thus, since all terms in the three sums in Equation \ref{eq:der_hs_norm} are either 0, if $s_i=0$, or negative, if $s_i>0$, the sums, and hence the derivatives, are negative: For $\lambda=0$,
\begin{equation*}
\begin{aligned}
&\frac{\partial}{\partial \lambda}\left\|\frac1n\I_n-\E(\ssv_\lambda\ssvt_\lambda)\right\|_{T}<0\\
&\frac{\partial}{\partial \lambda}\left\|\frac1n\I_n-\E(\ssv_\lambda\ssvt_\lambda)\right\|_{*}<0\\
&\frac{\partial}{\partial \lambda}\left\|\frac1n\I_n-\E(\ssv_\lambda\ssvt_\lambda)\right\|_{F}<0.
\end{aligned}
\end{equation*}

For the spectral norm, we again need $s_m>0$. If $\Ph\Ph^\top$ is non-singular, this is trivial, while if $\Ph\Ph^\top$ is singular, there exists, by assumption, at least one $s_i^2\in(0,n/2)$, for which $\frac1{s_i^2}-\frac1n>\frac{2}{n}-\frac1n=\frac1n$. Thus, since the maximum in Equation \ref{eq:hs_norm} is obtained for an $i$ such that $s_i>0$, we have $s_m>0$, and 

\begin{equation*}
\left.\frac{\partial}{\partial \lambda}\left\|\frac1n\I_n-\E(\ssv_\lambda\ssvt_\lambda)\right\|_{2}\right|_{\lambda=0}
=\frac{2}{s_m^4}\cdot\sgn\left(\frac1n-\frac{1}{s_m^2}\right).
\end{equation*}
Again, since $\|\Ph\Ph^\top\|_2< n$, $s_m^2\in(0,n)$ and we can write $s_m^2=n-a_m$ for $0< a_m< n$, and 
\begin{equation*}
\left(\frac1n-\frac{1}{s_m^2}\right)
=\left(\frac1n-\frac1{n-a_m}\right)=\left(-\frac{a_m}{n(n-a_m)}\right)<0.
\end{equation*}

\textbf{Gradient Flow ($t$)}\\
\emph{Part (a)}\\
For part \emph{(a)}, let $a_i:=e^{-ts_i^2}\geq 0$. Then, the singular values of 
$(\I_n-\Ss_t)^\top(\I_n-\Ss_t)$ and $\I_n-\Ss_t$ are $\left\{e^{-ts_i^4}\right\}_{i=1}^n=\left\{a_i^2\right\}_{i=1}^n$, and $\left\{e^{-ts_i^2}\right\}_{i=1}^n=\left\{a_i\right\}_{i=1}^n$ respectively. Since the $a_i\geq 0$ for all $i$, $\I_n-\Ss_t$ is positive semi-definite, which means that its singular and eigenvalues coincide and thus that its trace equals the sum of its singular values. Consequently

\begin{equation}
\label{eq:gcv_mses_gf}
\begin{aligned}
\frac{\left\|(\I_n-\Ss_t)^\top(\I_n-\Ss_t)\right\|}{\left(\Tr(\I_n-\Ss_t)\right)^2}
&=\begin{cases}
\frac{\sum_{i=1}^na_i^2}{\left(\sum_{i=1}^na_i\right)^2}\text{ for the nuclear (and trace semi) norm.}\\
\sqrt{\frac{\sum_{i=1}^na_i^4}{\left(\sum_{i=1}^na_i\right)^4}}\text{ for the Frobenius norm.}\\
\frac{\max_{i=1,\dots n}a_i^2}{\left(\sum_{i=1}^na_i\right)^2}\text{ for the spectral norm.}\\
\end{cases}
\end{aligned}
\end{equation}
For the nuclear and Frobenius norms, we can use that for any vector, $\vv\in \R^n$, for $d>1$,
\begin{equation*}
\|\vv\|_d=\left(\sum_{i=1}^n|v_i|^d\right)^{1/d}\geq n^{1-1/d}\cdot\|\vv\|_1=n^{1-1/d}\cdot\sum_{i=1}^n|v_i|,
\end{equation*}
with equality iff $v_i=1$ for all $i$. Thus the two corresponding expressions in Equation \ref{eq:gcv_mses_gf} are minimized when $a_i=1$ for all $i$, i.e.\ for $t=0$, in which case their values are

\begin{equation*}
\begin{aligned}
\lim_{t\to 0}\left(\frac{\left\|(\I_n-\Ss_t)^\top(\I_n-\Ss_t)\right\|_*}{\left(\Tr(\I_n-\Ss_t)\right)^2}\right)=\frac1n\\
\lim_{t\to 0}\left(\frac{\left\|(\I_n-\Ss_t)^\top(\I_n-\Ss_t)\right\|_F}{\left(\Tr(\I_n-\Ss_t)\right)^2}\right)=\frac1{n\sqrt{n}}.
\end{aligned}
\end{equation*}
For the spectral norm, let $a_m:=\max_{i=1,\dots n}a_i$. Then the corresponding expression in Equation \ref{eq:gcv_mses_gf} becomes $a_m^2/\left(\sum_{i=1}^na_i\right)^2=1/\left(\sum_{i=1}^n\frac{a_i}{a_m}\right)^2$. Here, $\frac{a_i}{a_m}\leq 1$ is maximized when $a_i=a_m$ for all $i$, i.e. when $t=0$ and $a_i=1$ for all $i$.
In this case,
\begin{equation*}
\lim_{t\to 0}\left(\frac{\left\|(\I_n-\Ss_t)^\top(\I_n-\Ss_t)\right\|_2}{\left(\Tr(\I_n-\Ss_t)\right)^2}\right)=\frac1{n^2}.
\end{equation*}

~\\

\emph{Part (b)}\\
For part \emph{(b)},
\begin{equation*}
\begin{aligned}
&n\cdot\left\|\frac1n\I_n-\frac1n\Ss^\top_t\Ss_t\right\|
=n\cdot\left\|\frac1n\I_n-\frac1n\left(\I_n-\exp(-t\Ph\Ph^\top)\right)^2\right\|\\
=&\begin{cases}
\sum_{i=1}^n\left|1-\left(1-e^{-ts_i^2}\right)^2\right|\text{ for the nuclear (and trace semi) norm.}\\
\sum_{i=1}^n\left(1-\left(1-e^{-ts_i^2}\right)^2\right)^2\text{ for the squared Frobenius norm.}\\
\max_{i=1,\dots n}\left|1-\left(1-e^{-ts_i^2}\right)^2\right|\text{ for the spectral norm.}\\
\end{cases}
\end{aligned}
\end{equation*}
If $s_i>0$, $e^{-ts_i^2}\in[0,1]$ for $t\geq 0$, and equals 0 only for $t=\infty$. Hence, $\left|1-\left(1-e^{-ts_i^2}\right)^2\right|$ (and thus $\left(1-\left(1-e^{-ts_i^2}\right)^2\right)^2$) is minimized for $t=\infty$, in which case it is 0.
If $s_i=0$, $\left|1-\left(1-e^{-ts_i^2}\right)^2\right|=1$ regardless of $t$. 
For the nuclear and Frobenius norms, since each term of the sum obtains its minimum for $t=\infty$, so does the sum.
For the spectral norm, the minimum is obtained either for $t=\infty$, if $\argmax_{s_i}\left|1-\left(1-e^{-ts_i^2}\right)^2\right|>0$, or for any value of $t$ if $\argmax_{s_i}\left|1-\left(1-e^{-ts_i^2}\right)^2\right|=0$.

Analogously,
\begin{equation*}
\begin{aligned}
n\cdot\left\|\frac1n\I_n-\frac1n\Ss_t\right\|
=&n\cdot\left\|\frac1n\I_n-\frac1n\left(\I_n-\exp(-t\Ph\Ph^\top)\right)\right\|\\
=&\begin{cases}
\sum_{i=1}^n\left|e^{-ts_i^2}\right|\text{ for the nuclear (and trace semi) norm.}\\
\sum_{i=1}^n\left(e^{-ts_i^2}\right)^2\text{ for the squared Frobenius norm.}\\
\max_{i=1,\dots n}\left|e^{-ts_i^2}\right|\text{ for the spectral norm.}\\
\end{cases}
\end{aligned}
\end{equation*}

where, for $t\geq 0$, $e^{-ts_i^2}\in[0,1]$ equals 0 only for $t=\infty$, if $s_i>0$,
and $e^{-ts_i^2}=1$ regardless of $t$ if $s_i=0$. 
Again, for the nuclear and Frobenius norms, each term of the sums obtains its minimum for $t=\infty$, and so do the sums.
For the spectral norm, the minimum is obtained either for $t=\infty$, if $\argmax_{s_i}e^{-ts_i}>0$, or for any value of $t$ if $\argmax_{s_i}e^{-ts_i}=0$.

~\\

\emph{Part (c)}\\
For part \emph{(c)}, we define $\tau=1/t$, and want to show that the derivative of $\left\|\frac1n\I-\E(\ssv_\tau\ssvt_\tau)\right\|$ with respect to $\tau$ is strictly negative for $\tau=0$, strictly positive for $M_1<\tau<M_2$ and non-negative for $\tau\geq M_2$, in which case the minimum of the norm is obtained for $\tau\in(0,M_1)$, i.e.\ for $t>1/M_1$.

Since
\begin{equation*}
\begin{aligned}
&\frac1n\I_n-\E(\ssv_\tau\ssvt_\tau)=\frac1n\I_n-\left(\I_n-\exp(-\Ph\Ph^\top/\tau)\right)(\Ph\Ph^\top)^{-1}\Ph\underbrace{\E(\phsv\phsvt)}_{=\I_p}\Ph^\top(\Ph\Ph^\top)^{-1}(\I_n-\exp(-\Ph\Ph^\top/\tau))\\
&=\frac1n\I_n-\left(\I_n-\exp(-\Ph\Ph^\top/\tau)\right)^2(\Ph\Ph^\top)^{-1},
\end{aligned}
\end{equation*}

\begin{equation}
\label{eq:hs_norm_gf} 
\begin{aligned}
&\left\|\frac1n\I_n-\E(\ssv_\tau\ssvt_\tau)\right\|_T=\left|\sum_{i=1}^n\left(\frac1n-\left(\frac{1-e^{-s_i^2/\tau}}{s_i}\right)^2\right)\right|\\
&\left\|\frac1n\I_n-\E(\ssv_\tau\ssvt_\tau)\right\|_*=\sum_{i=1}^n\left|\frac1n-\left(\frac{1-e^{-s_i^2/\tau}}{s_i}\right)^2\right|\\
&\left\|\frac1n\I_n-\E(\ssv_\tau\ssvt_\tau)\right\|_F=\sum_{i=1}^n\left(\frac1n-\left(\frac{1-e^{-s_i^2/\tau}}{s_i}\right)^2\right)^2\\
&\left\|\frac1n\I_n-\E(\ssv_\tau\ssvt_\tau)\right\|_2=\max_{i=1,\dots n}\left|\frac1n-\left(\frac{1-e^{-s_i^2/\tau}}{s_i}\right)^2\right|=\left|\frac1n-\left(\frac{1-e^{-s_m^2/\tau}}{s_i}\right)^2\right|,
\end{aligned}
\end{equation}
where $s_m:=\argmax_{s_i}\left|\frac1n-\left(\frac{1-e^{-ts_i^2}}{s_i}\right)^2\right|$.

The needed derivatives with respect to $\tau$ are
\begin{equation*}
\begin{aligned}
&\frac{\partial}{\partial \tau}\left(\frac1n-\left(\frac{1-e^{-s_i^2/\tau}}{s_i}\right)^2\right)
=\frac{2}{\tau^2}\cdot \left(1-e^{-s_i^2/\tau}\right)e^{-s_i^2/\tau}\\
&\frac{\partial}{\partial \tau}\left|\frac1n-\left(\frac{1-e^{-s_i^2/\tau}}{s_i}\right)^2\right|
=\frac{2}{\tau^2}\cdot \left(1-e^{-s_i^2/\tau}\right)e^{-s_i^2/\tau}\cdot\sgn\left(\frac1n-\left(\frac{1-e^{-s_i^2/\tau}}{s_i}\right)^2\right)\\
&\frac{\partial}{\partial \tau}\left(\frac1n-\left(\frac{1-e^{-s_i^2/\tau}}{s_i}\right)^2\right)^2
=\frac{4}{\tau^2}\cdot \left(1-e^{-s_i^2/\tau}\right)e^{-s_i^2/\tau}\cdot\left(\frac1n-\left(\frac{1-e^{-s_i^2/\tau}}{s_i}\right)^2\right)\\
\end{aligned}
\end{equation*}
and hence
\begin{equation}
\label{eq:der_hs_norm_gf}
\begin{aligned}
&\frac{\partial}{\partial \tau}\left\|\frac1n\I_n-\E(\ssv_\tau\ssvt_\tau)\right\|_T
=\sum_{i=1}^n\frac{2}{\tau^2}\cdot \left(1-e^{-s_i^2/\tau}\right)e^{-s_i^2/\tau}\cdot\sgn\left(\sum_{i=1}^n\left(\frac1n-\left(\frac{1-e^{-s_i^2/\tau}}{s_i}\right)^2\right)\right)\\
&\frac{\partial}{\partial \tau}\left\|\frac1n\I_n-\E(\ssv_\tau\ssvt_\tau)\right\|_{*}
=\sum_{i=1}^n\frac{2}{\tau^2}\cdot \left(1-e^{-s_i^2/\tau}\right)e^{-s_i^2/\tau}\cdot\sgn\left(\frac1n-\left(\frac{1-e^{-s_i^2/\tau}}{s_i}\right)^2\right)\\
&\frac{\partial}{\partial \tau}\left\|\frac1n\I_n-\E(\ssv_\tau\ssvt_\tau)\right\|_{F}
=\sum_{i=1}^n \frac{4}{\tau^2}\cdot \left(1-e^{-s_i^2/\tau}\right)e^{-s_i^2/\tau} \cdot\left(\frac1n-\left(\frac{1-e^{-s_i^2/\tau}}{s_i}\right)^2\right)\\
&\frac{\partial}{\partial \tau}\left\|\frac1n\I_n-\E(\ssv_\tau\ssvt_\tau)\right\|_{2}
=\frac{2}{\tau^2}\cdot \left(1-e^{-s_m^2/\tau}\right)e^{-s_m^2/\tau} \cdot\sgn\left(\frac1n-\left(\frac{1-e^{-s_m^2/\tau}}{s_i}\right)^2\right).
\end{aligned}
\end{equation}

We first consider large $\tau$'s.
For the trace seminorm and the nuclear and Frobenius norm, we note that $\tau>\frac{s_i^2}{\log\left(\sqrt{n}/(\sqrt n-s_i)\right)}\implies\left(\frac1n-\left(\frac{1-e^{-s_i^2/\tau}}{s_i}\right)^2\right)>0\implies\sum_{i=1}^n\left(\frac1n-\left(\frac{1-e^{-s_i^2/\tau}}{s_i}\right)^2\right)>0$, and hence, if $\max_{i=1,\dots n}\left(\frac{s_i^2}{\log\left(\sqrt{n}/(\sqrt n-s_i)\right)}\right)=:M_1<\tau<\infty$, each term in the three sums in Equation \ref{eq:der_hs_norm_gf} sum is strictly positive for $s_i>0$, and zero for $s_i=0$. Consequently, for $M_1<\tau<\infty$,
\begin{equation*}
\begin{aligned}
&\frac{\partial}{\partial \tau}\left\|\frac1n\I_n-\E(\ssv_\tau\ssvt_\tau)\right\|_{T}>0\\
&\frac{\partial}{\partial \tau}\left\|\frac1n\I_n-\E(\ssv_\tau\ssvt_\tau)\right\|_{*}>0\\
&\frac{\partial}{\partial \tau}\left\|\frac1n\I_n-\E(\ssv_\tau\ssvt_\tau)\right\|_{F}>0.\\
\end{aligned}
\end{equation*}
For the spectral norm, to obtain $\frac{\partial}{\partial \tau}\left\|\frac1n\I_n-\E(\ssv_\tau\ssvt_\tau)\right\|_{2}>0$ when $s_m^2/\log\left(\sqrt{n}/(\sqrt n-s_i)\right)=:M_m<\tau<\infty$, we also need $s_m>0$. If $\Ph\Ph^\top$ is non-singular, all its singular values are non-zero, and this is trivially fulfilled. However, if $\Ph\Ph^\top$ is singular, at least one $s_i$ is zero. We call this singular value $s_0$ (i.e. $s_0=0$) and note (by Taylor expanding $e^{-s_0^2/\tau}$) that $\left|\frac1n-\left(\frac{1-e^{-s_0^2/\tau}}{s_0}\right)^2\right|=\frac1n$ regardless of $\tau$.
Furthermore, if $\tau\geq s_i^2/\log\left(\sqrt{n}/(\sqrt n-s_i\sqrt 2)\right)$ for all $s_i$, then 
$\left|\frac1n-\left(1-e^{-s_i^2/\tau}\right)^2\right|\leq \frac1n$. This means that for for $\tau\geq \max_{i=1,\dots n}s_i^2/\log\left(\sqrt{n}/(\sqrt n-s_i\sqrt 2)\right)$, $s_m=0$ and $\left\|\frac1n\I_n-\E(\ssv_\tau\ssvt_\tau)\right\|_{2}=\frac1n$.
~\\

We now consider $\tau=0$.
For the trace seminorm and for the nuclear and Frobenius norms, if $s_i=0$, the corresponding terms in the sums in Equation \ref{eq:der_hs_norm_gf} are zero, and thus do not contribute when calculating the derivative. For terms where $s_i>0$, $\lim_{\tau \to 0}\left(\frac1n-\left(\frac{1-e^{-s_i^2/\tau}}{s_i}\right)^2\right)=\left(\frac1n-\frac{1}{s_i^2}\right)$.
Since $\|\Ph\Ph^\top\|_2< n$ by assumption, $s_i^2\in(0,n)$ for all $i$, and we can write $s_i^2=n-a_i$ for $0< a_i< n$. Hence,

\begin{equation*}
\left(\frac1n-\frac{1}{s_i^2}\right)
=\left(\frac1n-\frac1{n-a_i}\right)=\left(-\frac{a_i}{n(n-a_i)}\right)<0.
\end{equation*}
Thus, since all terms in the three sums in Equation \ref{eq:der_hs_norm_gf} are either 0, if $s_i=0$, or negative, if $s_i>0$, the sums, and hence the derivatives, are negative: For $\tau=0$,
\begin{equation*}
\begin{aligned}
&\frac{\partial}{\partial \tau}\left\|\frac1n\I_n-\E(\ssv_\tau\ssvt_\tau)\right\|_{T}<0\\
&\frac{\partial}{\partial \tau}\left\|\frac1n\I_n-\E(\ssv_\tau\ssvt_\tau)\right\|_{*}<0\\
&\frac{\partial}{\partial \tau}\left\|\frac1n\I_n-\E(\ssv_\tau\ssvt_\tau)\right\|_{F}<0.
\end{aligned}
\end{equation*}

For the spectral norm, we again need $s_m>0$. If $\Ph\Ph^\top$ is non-singular, this is trivial, while if $\Ph\Ph^\top$ is singular, there exists, by assumption, at least one $s_i^2\in(0,n/2)$, for which $\frac1{s_i^2}-\frac1n>\frac{2}{n}-\frac1n=\frac1n$. Thus, since the maximum in Equation \ref{eq:hs_norm_gf} is obtained for an $i$ such that $s_i>0$, we have $s_m>0$, and 

\begin{equation*}
\left.\frac{\partial}{\partial \tau}\left\|\frac1n\I_n-\E(\ssv_\tau\ssvt_\tau)\right\|_{2}\right|_{\tau=0}
=\frac{2}{s_m^4}\cdot\sgn\left(\frac1n-\frac{1}{s_m^2}\right).
\end{equation*}
Again, since $\|\Ph\Ph^\top\|_2< n$, $s_m^2\in(0,n)$ and we can write $s_m^2=n-a_m$ for $0< a_m< n$, and 
\begin{equation*}
\left(\frac1n-\frac{1}{s_m^2}\right)
=\left(\frac1n-\frac1{n-a_m}\right)=\left(-\frac{a_m}{n(n-a_m)}\right)<0.
\end{equation*}

\end{proof}

\begin{proof}[Proof of Theorem \ref{thm:snn_k}]~\\
For gradient descent with momentum, $\thetahv$ is updated according to
\begin{equation*}
\begin{aligned}
&\thetahv_{-1}=\thetahv_0,\\
&\thetahv_{k+1}=\thetahv_k+\gamma\cdot\left(\thetahv_k-\thetahv_{k-1}\right)-\eta\cdot\frac{\partial L(\fhv_k,\yv)}{\partial \thetav_k},\ k=0,1,\dots,
\end{aligned}
\end{equation*}
where, according to the chain rule,
\begin{equation*}
\frac{\partial L(\fhv_k,\yv)}{\partial \thetav_k}
=\left(\frac{\partial \fhv_k}{\partial \thetav_k}\right)^\top\frac{\partial L(\fhv_k,\yv)}{\partial \fhv_k}
=\left(\frac{\partial \fhv_k}{\partial \thetav_k}\right)^\top\IFk\cdot(\fhv_k-\yv),
\end{equation*}
where $\IFk=\I_n$ for the squared loss and $\IFk=\bm{\tilde{F}}_k$ for the cross-entropy error (see Equation \ref{eq:F_cross_entr}).
That is, $\thetahv$ is updated according to
\begin{equation*}
\label{eq:theta_upd}
\begin{aligned}
&\thetahv_{-1}=\thetahv_0,\\
&\thetahv_{k+1}=\thetahv_k+\gamma\cdot\left(\thetahv_k-\thetahv_{k-1}\right)+\eta\left(\frac{\partial \fhv_k}{\partial \thetav_k}\right)^\top\IFk\cdot(\yv-\fhv_k),\ k=0,1,\dots.
\end{aligned}
\end{equation*}
The proof consists of two steps: First, we define the function $\fhv'^\star_k$ according to 
\begin{equation*}
\begin{aligned}
&\fhv'^\star_{-1}=\fhv'^\star_0=\fhv^\star(\thetahv_0)\\
&\fhv'^\star_{k+1}=\fhv'^\star_k+\gamma\cdot\left(\fhv'^\star_k-\fhv'^\star_{k-1}\right)+\eta\cdot\Kt^\star_{k+1}\cdot\left(\yv-\fhv'_k\right),\ k=0,1,\dots
\end{aligned}
\end{equation*}
and show that
\begin{equation*}
\begin{aligned}
&\Ss^\star_{-1}=\Ss^\star_0=\nv\\
&\Ss^\star_{k+1}=\Ss^\star_k+\gamma\cdot\left(\Ss^\star_{k}-\Ss^\star_{k-1}\right)+\eta\cdot \Kt^\star_{k+1}\cdot\left(\I_n-\Ss_{k}\right),\ k=0,1,\dots
\end{aligned}
\end{equation*}
implies that 
\begin{equation}
\label{eq:fpH}
\fhv'^\star_k=\Ss^\star_k\left(\yv-\fhv(\thetahv_0)\right)+\fhv^\star(\thetahv_0).
\end{equation}

Then, we show that there exists a $C<\infty$, such that 
\begin{equation}
\label{eq:ffp}
\left\|\fhv^\star(\thetahv_k)-\fhv'^\star_k \right\|_\infty\leq \eta \cdot C.
\end{equation}
To alleviate notation, we will write $\fhv^\star_k:=\fhv^\star(\thetahv_k)$.

We first show Equation \ref{eq:fpH} by induction:

For $k=0$,
\begin{equation*}
\begin{aligned}
&\Ss^\star_0=\nv\\
&\fhv'^\star_0=\fhv_0=\nv=\nv\cdot\left(\yv-\nv\right)+\nv=\Ss^\star_0\cdot\left(\yv-\fhv_0\right)+\fhv^\star_0.
\end{aligned}
\end{equation*}
For $k=1$,
\begin{equation*}
\begin{aligned}
\Ss^\star_1&=
\Ss^\star_0+\gamma\cdot\left(\Ss^\star_0-\Ss^\star_{-1}\right)+\eta\cdot\Kt^\star_1\left(\I_n-\Ss_0\right)
=\nv+\gamma\cdot\left(\nv-\nv\right)+\eta\cdot\Kt^\star_1\left(\I_n-\nv\right)=\eta\cdot\Kt^\star_1\\
\fhv'^\star_1&=\fhv'^\star_0+\gamma\cdot\left(\fhv'^\star_0-\fhv'^\star_{-1}\right)+\eta\cdot\Kt^\star_1\left(\yv-\fhv'_0\right)
=\fhv^\star_0+\gamma\cdot\left(\fhv^\star_0-\fhv^\star_0\right)+\eta\cdot\Kt^\star_1\left(\yv-\fhv_0\right)\\
&=\Ss^\star_1\left(\yv-\fhv_0\right)+\fhv^\star_0.
\end{aligned}
\end{equation*}
For $k+1\geq 2$,

\begin{equation*}
\begin{aligned}
\fhv'^\star_{k-1}=&\Ss^\star_{k-1}\left(\yv-\fhv_0\right)+\fhv^\star_0\\
\fhv'^\star_k=&\Ss^\star_k\left(\yv-\fhv_0\right)+\fhv^\star_0\\
\Ss^\star_{k+1}=&\Ss^\star_k+\gamma\cdot\left(\Ss^\star_k-\Ss^\star_{k-1}\right)+\eta\cdot \Kt^\star_{k+1}\cdot\left(\I_n-\Ss_k\right)\\
\fhv'^\star_{k+1}=&\fhv'^\star_k+\gamma\cdot\left(\fhv'^\star_k-\fhv'^\star_{k-1}\right)+\eta\cdot\Kt^\star_{k+1}\left(\yv-\fhv'_k\right)\\
=&\Ss^\star_k\left(\yv-\fhv_0\right)+\fhv^\star_0+\gamma\cdot\left(\Ss^\star_k\left(\yv-\fhv_0\right)+\fhv^\star_0-\Ss^\star_{k-1}\left(\yv-\fhv_0\right)-\fhv^\star_0\right)\\
&+\eta\cdot\Kt^\star_{k+1}\left(\yv-\Ss^\star_k\left(\yv-\fhv_0\right)-\fhv^\star_0\right)\\
=&\left(\Ss^\star_k+\gamma\cdot\left(\Ss^\star_k-\Ss^\star_{k-1}\right)\right)\left(\yv-\fhv_0\right)+\eta\cdot\Kt^\star_{k+1}\left(\I_n-\Ss_k\right)\left(\yv-\fhv_0\right)+\fhv^\star_0\\
=&\Ss^\star_{k+1}\left(\yv-\fhv_0\right)+\fhv^\star_0.
\end{aligned}
\end{equation*}

The proof of Equation \ref{eq:ffp} is to a large extent based on the mean value inequality,
\begin{equation}
\label{eq:mvi}
\left\|\bm{f}(x+\Delta x)-\bm{f}(x)\right\|_2\leq \max_x\left\|\frac{\partial\bm{f}(x)}{\partial x}\right\|_2\cdot |\Delta x|,
\end{equation}
and the second-order Taylor expansion,
\begin{equation}
\label{eq:tay}
f(\xv+\bm{\Delta}\xv)=f(\xv)+\bm{\Delta}\xv^\top\left(\frac{\partial f(\xv)}{\partial \xv}\right)+r,\  |r|\leq \frac12\|\bm{\Delta}\xv\|_2^2\cdot \max_{\xv}\left\|\frac{\partial^2f(\xv)}{\partial\xv^2}\right\|_2.
\end{equation}

We first show that 
\begin{equation}
\label{eq:fk1}
\begin{aligned}
&\fhv^\star_{k+1} =\fhv^\star_k+\gamma\cdot\left(\fhv^\star_k-\fhv^\star_{k-1}\right)+\eta\cdot\Kt^\star_{k+1}\cdot\left(\yv-\fhv^\star_k\right)+\eta^2\cdot\bm{r}\\
&\text{where } \|\bm{r}\|_\infty\leq M_1
:=\frac12M_{f_{\theta^2}}\left(\gamma M_{\theta_t}^2 +  \left(\gamma M_{\theta_t}+M_y M_{f_\theta}\right)^2\right),
\end{aligned}
\end{equation}
for
\begin{equation*}
\begin{aligned}
&M_{\theta_t}:=\max_t\left\|\frac{\partial\thetahv(t)}{\partial t}\right\|_2,\ 
M_{f_\theta}:=\max_t\left\|\frac{\partial\fhv^\star(t)}{\partial\thetahv}\right\|_2,\ 
M_{f_{\theta^2}}:=\max_t\left\|\frac{\partial^2\fh^\star(t)}{\partial \thetahv^2}\right\|_2,\\ 
&M_y:=\max_t\left\|\yv-\fhv^\star_k\right\|_2.
\end{aligned}
\end{equation*}

To keep track of $\eta$, we switch between reminder terms of the form $\tilde{r}$, which include $\eta$, and $r$, which do not, i.e.\ $\tilde{r}=\eta r$.

By applying Equation \ref{eq:mvi} to $\thetahv(t)$, where $\thetahv_k=\thetahv(k\eta)$, we obtain
\begin{equation*}
\begin{aligned}
\left\|\thetahv_k-\thetahv_{k-1}\right\|_2=\left\|\thetahv\left(k\eta\right)-\thetahv\left((k-1)\eta\right)\right\|_2\leq \max_t\left\|\frac{\partial \thetahv(t)}{\partial t}\right\|_2\cdot |k\eta-(k-1)\eta|=M_{\theta_t}\cdot\eta
\end{aligned}
\end{equation*}

Applying Equation \ref{eq:tay} element-wise to $\fhv^\star_k$, in combination with Equation \ref{eq:mvi}, we obtain
\begin{equation*}
\begin{aligned}
\fh^\star_k-\fh^\star_{k-1}&=
\fh^\star\left(\thetahv_k\right)-\fh^\star\left(\thetahv_{k-1}\right)
=\fh^\star\left(\thetahv_k\right)-\fh^\star\left(\thetahv_k-(\thetahv_k-\thetahv_{k-1})\right)\\
&=\fh^\star\left(\thetahv_k\right)-\left(\fh^\star\left(\thetahv_k\right)-\left(\thetahv_k-\thetahv_{k-1}\right)^\top\left(\frac{\partial\fh^\star_k}{\partial \thetahv}\right)+\tilde{r}_1\right)\\
&=\left(\thetahv_k-\thetahv_{k-1}\right)^\top\left(\frac{\partial\fh^\star_k}{\partial \thetahv_k}\right)-\tilde{r}_1
\iff\left(\thetahv_k-\thetahv_{k-1}\right)^\top\left(\frac{\partial \fh^\star_k}{\partial \thetahv_k}\right)=\fh^\star_k-\fh^\star_{k-1}+\tilde{r}_1,\\
|\tilde{r}_1|&\leq\frac12\left\|\thetahv_k-\thetahv_{k-1}\right\|_2^2\cdot \max_t\left\|\frac{\partial^2 \fh^\star(t)}{\partial \thetahv_k^2}\right\|_2\leq\frac12\eta^2 M_{\theta_t}^2 M_{f_{\theta^2}}\\
\iff& \left(\thetahv_k-\thetahv_{k-1}\right)^\top\left(\frac{\partial\fh^\star_k}{\partial \thetahv_k}\right)=\fh^\star_k-\fh^\star_{k-1}+\eta^2\cdot r_1,\ |r_1|\leq\frac12M_{\theta_t}^2M_{f_{\theta^2}}.
\end{aligned}
\end{equation*}

With $\kvt^\star_k:=\IFk^\top\cdot\frac{\partial\fhv^\star_k}{\partial \thetahv_k}\cdot\frac{\partial\fh^\star_k}{\partial\thetahv_k}$, we obtain

\begin{equation*}
\begin{aligned}
\fh^\star_{k+1}=&
\fh^\star\left(\thetahv_{k+1}\right)
=\fh^\star\left(\thetahv_k+\gamma\cdot\left(\thetahv_k-\thetahv_{k-1}\right)+\eta\cdot\left(\frac{\partial\fhv^\star_k}{\partial \thetahv_k}\right)^\top\IFk\left(\yv-\fhv_k\right)\right)\\
=&\fh^\star\left(\thetahv_k\right)+\gamma\cdot\left(\thetahv_k-\thetahv_{k-1}\right)^\top\left(\frac{\partial\fh^\star_k}{\partial \thetahv_k}\right)+\left(\eta\cdot\left(\frac{\partial \fhv^\star_k}{\partial\thetahv_k}\right)^\top\IFk\left(\yv-\fhv_k\right)\right)^\top\left(\frac{\partial \fh^\star_k}{\partial\thetahv_k}\right) +\tilde{r}_2\\
=&\fh^\star_k+\gamma\cdot\left(\fh^\star_k-\fh^\star_{k-1}\right)+\gamma\cdot \eta^2\cdot r_1+\eta\cdot\left(\yv-\fhv_k\right)^\top\kvt^\star_k+\tilde{r}_2\\
=&\fh^\star_k+\gamma\cdot\left(\fh^\star_k-\fh^\star_{k-1}\right)+\eta\cdot\kvt^\star_k{^\top}\left(\yv-\fhv_k\right)+\underbrace{\tilde{r}}_{=:\gamma\cdot \eta^2\cdot r_1+\tilde{r}_2},\\
|\tilde{r}_2|\leq &\frac12\cdot \left\|\gamma\cdot\left(\thetahv_k-\thetahv_{k-1}\right)+\eta\cdot\left(\frac{\partial \fhv^\star_k}{\partial\thetahv_k}\right)^\top\left(\yv-\fhv_k\right)\right\|^2\cdot M_{f_{\theta^2}}\\
\leq&\frac12\cdot\left(\gamma^2\cdot\left\|\thetahv_k-\thetahv_{k-1}\right\|^2+2\gamma\cdot\left\|\thetahv_k-\thetahv_{k-1}\right\|\cdot\eta \left\|\frac{\partial\fhv^\star_k}{\partial\thetahv_k}\right\| \cdot\left\|\yv-\fhv_k\right\|\right.\\
&\left.+\eta^2 \left\|\frac{\partial \fhv^\star_k}{\partial\thetahv_k}\right\|^2\cdot\left\|\yv-\fhv_k\right\|^2\right)\cdot M_{f_{\theta^2}}\\
\leq&\frac12\left(\gamma^2\eta^2M_{\theta_t}^2+2\gamma\eta^2 M_{\theta_t} M_{f_\theta}M_y+\eta^2 M_{f_\theta}^2M_y^2\right)\cdot M_{f_{\theta^2}}\\
=&\frac12\eta^2M_{f_{\theta^2}}\left(\gamma^2M_{\theta_t}^2+2\gamma M_{\theta_t} M_{f_\theta}M_y+M_{f_\theta}M_y^2\right)
=\frac12\eta^2M_{f_{\theta^2}}\left(\gamma M_{\theta_t}+M_y M_{f_\theta}\right)^2\\
|\tilde{r}|\leq&\eta^2\cdot\gamma\cdot|r_1|+|\tilde{r}_2|
\leq\eta^2\cdot\gamma\cdot\frac12M_{\theta_t}^2 M_{f_{\theta^2}}+\frac12\eta^2M_{f_{\theta^2}}\left(\gamma M_{\theta_t}+M_y M_{f_\theta}\right)^2\\
=&\frac{\eta^2}2\cdot M_{f_{\theta^2}}\left(\gamma M_{\theta_t}^2 +  \left(\gamma M_{\theta_t}+M_y M_{f_\theta}\right)^2\right)\\
\iff&\fh^\star_{k+1}=\fh^\star_k+\gamma\cdot(\fh^\star_k-\fh^\star_{k-1})+\eta\cdot\kvt^\star_k{^\top}\left(\yv-\fhv_k\right)+\eta^2\cdot r,\\
&|r|\leq\frac{\eta^2}2\cdot M_{f_{\theta^2}}\left(\gamma M_{\theta_t}^2 +  \left(\gamma M_{\theta_t}+M_y M_{f_\theta}\right)^2\right)\\
\implies&\fhv^\star_{k+1} =\fhv^\star_k+\gamma\cdot\left(\fhv^\star_k-\fhv^\star_{k-1}\right)+\eta\cdot\Kt^\star_k\left(\yv-\fhv_k\right)+\eta^2\cdot\bm{r}\\
&\|\bm{r}\|_\infty\leq\frac{\eta^2}2\cdot M_{f_{\theta^2}}\left(\gamma M_{\theta_t}^2 +  \left(\gamma M_{\theta_t}+M_y M_{f_\theta}\right)^2\right)=:\eta^2\cdot M_1,
\end{aligned}
\end{equation*}
which proves Equation \ref{eq:fk1}.

To prove Equation \ref{eq:ffp}, we first apply Equation \ref{eq:tay} element-wise to $\fhv^\star_k$, to obtain
\begin{equation*}
\begin{aligned}
&\left(\fh^\star_{k+1}-\fh^\star_k\right)-\left(\fh^\star_k-\fh^\star_{k-1}\right)=
\left(\fh^\star\left((k+1)\eta\right)-\fh^\star\left(k\eta\right)\right)-\left(\fh^\star\left(k\eta\right)-\fh^\star\left((k-1)\eta\right)\right)\\
&=\left(\fh^\star\left(k\eta\right)+\eta\cdot\partial_t\fh^\star\left(k\eta\right)+ \tilde{r}_{f_1}-\fh^\star\left(k\eta\right)\right)
-\left(\fh^\star\left(k\eta\right)-\fh^\star\left(k\eta\right)+\eta\cdot\partial_t\fh^\star\left(k\eta\right)-\tilde{r}_{f_2}\right)\\
&=\tilde{r}_{f_1}+\tilde{r}_{f_2}=:\tilde{r}_f,\ |\tilde{r}_f|\leq 2\cdot \frac12\eta^2\cdot\max_t\left|\frac{\partial^2\fh^\star(t)}{\partial t^2}\right|\\
\iff&\fh^\star_k-\fh^\star_{k-1}=\fh^\star_{k+1}-\fh^\star_k+\eta^2\cdot r_f,\ |r_f|\leq M_{f_t^2}\\
\implies&\fhv^\star_k-\fhv^\star_{k-1}=\fhv^\star_{k+1}-\fhv^\star_k+\eta^2\cdot \bm{r_f},\ \|\bm{r_f}\|_\infty\leq M_{f_t^2},
\end{aligned}
\end{equation*}
where
\begin{equation*}
M_{f_t^2}:=\max\left(\max_t\left|\frac{\partial^2 \fh^\star(t)}{\partial t^2}\right|,\ \max_t\left|\frac{\partial^2 \fh'^\star(t)}{\partial t^2}\right|\right).
\end{equation*}
Thus,
\begin{equation}
\label{eq:ek}
\begin{aligned}
\fhv^\star_{k+1}-\fhv'^\star_{k+1}
=&\fhv^\star_k+\gamma\cdot(\fhv^\star_k-\fhv^\star_{k-1})+\eta\cdot\Kt^\star_k\left(\yv-\fhv_k\right)+\eta^2\cdot\bm{r}\\
&-\fhv'^\star_k-\gamma\cdot(\fhv'^\star_k-\fhv'^\star_{k-1})-\eta\cdot\Kt^\star_k\left(\yv-\fhv'_k\right)\\
=&\fhv^\star_k+\gamma\cdot(\fhv^\star_{k+1}-\fhv^\star_k)+\eta\cdot\Kt^\star_k\left(\yv-\fhv_k\right)\\
&-\fhv'^\star_k-\gamma\cdot(\fhv'^\star_{k+1}-\fhv'^\star_k)-\eta\cdot\Kt^\star_k\left(\yv-\fhv'_k\right)+\eta^2\cdot\underbrace{\left(\bm{r_f}+\bm{r}-\bm{r_{f'}}\right)}_{=:\bm{r_3}}\\
=&\fhv^\star_k-\fhv'^\star_k+\gamma\cdot\left(\left(\fhv^\star_{k+1}-\fhv'^\star_{k+1}\right)-\left(\fhv^\star_k-\fhv'^\star_k\right)\right)-\eta\cdot \Kt^\star_k\left(\fhv^\star_k-\fhv'^\star_k\right) + \eta^2\cdot\bm{r_3}\\
\|\bm{r_3}\|_\infty \leq&\eta^2\cdot M_1+\eta^2\cdot 2M_{f_t^2} =:\eta^2\cdot \tilde{C}.
\end{aligned}
\end{equation}

For $\ehv_k:=\fhv^\star_k-\fhv'^\star_k$, and $\eh_k:=\|\ehv_k\|_\infty$, we obtain, from Equation \ref{eq:ek},
\begin{equation*}
\begin{aligned}
\ehv_{k+1}&=\ehv_k+\gamma\cdot\left(\ehv_{k+1}-\ehv_k\right)-\eta\cdot\Kt^\star_k\cdot\ehv_k + \eta^2\cdot\bm{r_3}\\
(1-\gamma)\eh_{k+1}&\leq(1-\gamma)\eh_k+\eta\cdot M_{f_\theta}^2\cdot \eh_k + \eta^2\tilde{C}\\
\eh_{k+1}&\leq\left(1+\frac\eta{1-\gamma}\cdot M_{f_\theta}^2\right)\eh_k + \frac{\eta^2}{1-\gamma}\tilde{C}.
\end{aligned}
\end{equation*}
Since, by definition, $\eh_0=\left\|\fhv^\star_0-\fhv'^\star_0\right\|_\infty=\left\|\fhv^\star_0-\fhv^\star_0\right\|_\infty=0$, recursively applying this inequality above renders the  geometric series

\begin{equation*}
\begin{aligned}
\eh_k&\leq \sum_{l=0}^k\left(1+\frac\eta{1-\gamma}\cdot M_{f_\theta}^2\right)^l\cdot\frac{\eta^2}{1-\gamma}\tilde{C}
=\frac{\left(1+\frac\eta{1-\gamma}\cdot M_{f_\theta}^2\right)^k-1}{1+\frac\eta{1-\gamma}\cdot M_{f_\theta}^2-1}\cdot\frac{\eta^2}{1-\gamma}\tilde{C}\\
&=\frac{\left(1+\frac\eta{1-\gamma}\cdot M_{f_\theta}^2\right)^k-1}{M_{f_\theta}^2}\cdot\eta \tilde{C}
\leq\frac{e^{\frac1{1-\gamma}\cdot M_{f_\theta}^2\cdot \eta k}}{M_{f_\theta}^2}\cdot\eta \tilde{C}
=:\eta\cdot C,
\end{aligned}
\end{equation*}
for
\begin{equation*}
C:=\frac{e^{\frac1{1-\gamma}\cdot M_{f_\theta}^2\cdot \eta k}}{M_{f_\theta}^2}\cdot\left(M_1+2M_{f_t^2}\right),
\end{equation*}
where in the last inequality we have used that for $x>-1$, $(1+x)^k\leq e^{xk}$.

\end{proof}

\newpage

\begin{proof}[Proof of Proposition \ref{thm:cross_entr}]~\\
Let $\oslash$ denote element-wise division. Then
\begin{equation*}
\begin{aligned}
\frac{\partial \Lt(\fhiv,\yiv)}{\partial\fh_{ij}}
=&-\frac{\partial}{\partial \fh_{ij}}\left(\sum_{j=1}^{c-1}y_{ij}\cdot\log(\fh_{ij})+\left(1-\sum_{j=1}^{c-1}y_{ij}\right)\cdot\log\left(1-\sum_{j=1}^{c-1}\fh_{ij}\right)\right)\\
=&-\left(\frac{y_{ij}}{\fh_{ij}}-\frac{1-\sum_{j=1}^{c-1}y_{ij}}{1-\sum_{j=1}^{c-1}\fh_{ij}}\right)
\implies \frac{\partial L(\fhiv)}{\partial \fhiv}=\frac{1-\sum_{j=1}^{c-1}y_{ij}}{1-\sum_{j=1}^{c-1}\fh_{ij}}\cdot\bm{1}-\yiv\oslash\fhiv\\
=&\frac{1-\sum_{j=1}^{c-1}y_{ij}-\left(1-\sum_{j=1}^{c-1}\fh_{ij}\right)}{1-\sum_{j=1}^{c-1}\fh_{ij}}\cdot\bm{1}+\frac{1-\sum_{j=1}^{c-1}\fh_{ij}}{1-\sum_{j=1}^{c-1}\fh_{ij}}\cdot\bm{1}-\yiv\oslash\fhiv\\
=&\frac{\sum_{j=1}^{c-1}\fh_{ij}-\sum_{j=1}^{c-1}y_{ij}}{1-\sum_{j=1}^{c-1}\fh_{ij}}\cdot\bm{1}+\bm{1}-\yiv\oslash\fhiv \\
&=\frac{\bm{1}^\top(\fhiv-\yiv)}{1-\sum_{j=1}^{c-1}\fh_{ij}}\cdot\bm{1}+\fhiv\oslash\fhiv-\yiv\oslash\fhiv\\
=&\frac{1}{1-\sum_{j=1}^{c-1}\fh_{ij}}\cdot\bm{1}\bm{1}^\top(\fhiv-\yiv)+(\Fhi)^{-1}(\fhiv-\yiv)\\
=&\left((\Fhi)^{-1}+\frac1{1-\sum_{j=1}^{c-1}\fh_{ij}}\cdot\bm{1}\bm{1}^\top\right)\left(\fhiv-\yiv\right).
\end{aligned}
\end{equation*}
\end{proof}

\end{document}